\documentclass{article} 
\usepackage{iclr2025_conference,times}

\usepackage{amsmath,amsfonts,bm}

\def\eqref#1{equation~\ref{#1}}

\def\1{\bm{1}}

\DeclareMathAlphabet{\mathsfit}{\encodingdefault}{\sfdefault}{m}{sl}
\SetMathAlphabet{\mathsfit}{bold}{\encodingdefault}{\sfdefault}{bx}{n}

\usepackage{hyperref}
\renewcommand{\eqref}[1]{(\ref{#1})}
\usepackage{url}
\usepackage{amsmath,amssymb,amsthm}
\usepackage{booktabs}
\usepackage{multirow}
\usepackage{graphicx}
\graphicspath{{figures/}}
\usepackage{xcolor}

\newtheorem{theorem}{Theorem}

\newtheorem{proposition}{Proposition}

\newtheorem{definition}{Definition}
\newtheorem{remark}{Remark}

\title{Beyond Importance Sampling: \\ Rejection-Gated Policy Optimization}

\iclrfinalcopy   

\author{
Ziwu Sun$^{*}$ \\
\texttt{1552761350@qq.com} \\
\And
Zhen Gao$^{*}$ \\
\texttt{gaozhen0816@gmail.com} \\
\And
Jiyong Zhang$^{*}$ \\
\texttt{jiyongzhang@gmail.com} \\
\And
Jiaheng Li$^{*}$ \\
\texttt{ljh010215@163.com} 
}

\begin{document}

\maketitle

\let\thefootnote\relax\footnotetext{$^*$Equal contribution, listed in no particular order. \quad $^\dagger$Corresponding author.}

\begin{abstract}
We propose a new perspective on policy optimization: rather than
\emph{reweighting} all samples by their importance ratios, an optimizer should
\emph{select} which samples are trustworthy enough to drive a policy update.
Building on this view, we introduce \textbf{Rejection-Gated Policy Optimization
(RGPO)}, which replaces the importance-sampling ratio $r_\theta =
\pi_\theta/\pi_{\text{old}}$ with a smooth, differentiable \emph{acceptance
gate} $\alpha_\theta(s,a) = g(r_\theta(s,a)) \in [0,1]$.
Unlike prior work that applies rejection sampling as a data-level
heuristic before training, RGPO \emph{elevates rejection to an optimization
principle}: the gate participates directly in gradient computation and is
implicitly updated alongside the policy.
RGPO provides a unified framework: the policy gradients of TRPO, PPO, and
REINFORCE all correspond to specific choices of the effective gradient weight
$w(r)=g'(r)\cdot r$.
We prove that RGPO guarantees finite, bounded gradient variance even when
importance-sampling ratios are heavy-tailed (where IS variance diverges).
We further show that RGPO incurs only a bounded, controllable bias and
provides an approximate monotonic policy improvement guarantee analogous to
TRPO.
RGPO matches PPO in computational cost, requires no second-order optimization,
and extends naturally to RLHF-style preference alignment.
In online preference fine-tuning of \texttt{Qwen2.5-1.5B-Instruct} on
Anthropic HH-RLHF ($n\!=\!3$ seeds), RGPO uses a \emph{dual-ratio gate} that
anchors learning to both the previous policy and the reference model, achieving
a Pareto-dominant outcome: the highest reward among online RL methods
($+14.8\%$ vs.\ PPO-RLHF) \emph{and} the lowest KL divergence to the reference
model ($-16.0\%$ vs.\ PPO-RLHF, $-53.1\%$ vs.\ GRPO).
\end{abstract}

\section{Introduction}

\emph{How should a policy optimizer treat samples that were collected under a
different, older policy?}  This question sits at the heart of modern deep
reinforcement learning, and the answer almost universally given is:
\emph{reweight them} via importance sampling (IS).  Every sample is kept;
off-distribution samples are compensated for by multiplying their gradient
contribution by the density ratio $r_\theta = \pi_\theta / \pi_\text{old}$.

\paragraph{The problem with reweighting.}
Reweighting is mathematically sound but practically fragile.  When the current
policy $\pi_\theta$ drifts away from the behavior policy $\pi_\text{old}$,
importance ratios can become very large, causing gradient variance to explode
and training to destabilize.  PPO~\citep{schulman2017proximal} addresses this
by replacing the raw ratio with a clipped surrogate:
\begin{equation*}
  \mathcal{L}_\text{PPO}(\theta)
  = \mathbb{E}\!\left[
      \min\!\left(r_\theta A,\;
        \mathrm{clip}(r_\theta,\,1\!-\!\epsilon,\,1\!+\!\epsilon)\cdot A
      \right)
    \right].
\end{equation*}
The $\min$ creates an asymmetric mechanism: when $r_\theta$ falls \emph{outside}
$[1-\epsilon, 1+\epsilon]$ and $A > 0$, the clipped term is smaller, so the
$\min$ selects the clipped value and the gradient with respect to $\theta$ from
the $r_\theta A$ term is \emph{blocked} (the clipped constant carries no
gradient); when $A < 0$ and $r_\theta$ is out of range, the $r_\theta A$ term
is already the $\min$ and gradients flow normally, allowing the policy to
correct itself.  In short, PPO \emph{prevents over-optimistic updates} but
does not zero out all gradients for extreme samples—it is a one-sided guard
rather than a hard rejection.  Nevertheless, clipping is a discontinuous,
hand-designed heuristic that offers limited design flexibility and no principled
generalization to other forms of selection.

More fundamentally, clipping reveals an implicit assumption: \textbf{some
samples should not contribute to the update at all}.  But PPO enforces this
through an ad hoc threshold, not through a principled selection mechanism.

\paragraph{A missing abstraction: sample selection.}
We argue that the policy optimization literature has long lacked a principled
treatment of \emph{sample selection}—the question of \emph{whether} a sample
should influence an update, as opposed to \emph{how much} weight it should
receive.  Existing methods conflate the two:

\begin{itemize}
  \item \textbf{TRPO/PPO} (and IS in general): every sample is used; weight
        determines influence.
  \item \textbf{RAFT, RLHF rejection sampling}~\citep{gulcehre2023reinforced}:
        some samples are discarded \emph{before} training as a preprocessing
        step, outside the optimization loop.
\end{itemize}

Neither approach makes selection a \emph{differentiable, first-class component
of the optimization objective}.

\paragraph{Our proposal: differentiable selection.}
We propose \textbf{Rejection-Gated Policy Optimization (RGPO)}, which
introduces a smooth \emph{acceptance gate}
$\alpha_\theta(s,a) = g(r_\theta(s,a)) \in [0,1]$ directly into the surrogate
objective:
\begin{equation*}
  \mathcal{L}_\text{RGPO}(\theta)
  = \mathbb{E}_{\pi_\text{old}}\!\left[\alpha_\theta(s,a)\cdot A_\text{old}(s,a)\right].
\end{equation*}
The key properties of this design are:
\begin{itemize}
  \item \textbf{Differentiable.}  Gradients flow through $\alpha_\theta$ back
        into $\theta$ via the chain rule.  The gate is not a preprocessing step;
        it is part of the loss.
  \item \textbf{Implicit.}  No separate gate network or extra parameters are
        needed.  The gate is fully determined by the policy ratio $r_\theta$,
        and it updates automatically as $\theta$ evolves.
  \item \textbf{Unifying.}  The policy gradients of TRPO, PPO, and REINFORCE
        all fit the RGPO template via specific choices of the effective gradient
        weight $w(r)=g'(r)\cdot r$ (Table~\ref{tab:unified}): REINFORCE is an
        \emph{exact special case} in the on-policy limit ($r\equiv1$); TRPO and PPO are
        \emph{gradient-level correspondences} ($g\notin[0,1]$); AWR is a closely
        related method.
\end{itemize}

This constitutes what we call \emph{elevating rejection sampling from a
data-level heuristic to a differentiable optimization principle}.

\paragraph{Contributions.}
\begin{enumerate}
  \item \textbf{Framework.}  We introduce RGPO, replacing the IS ratio with a
        differentiable acceptance gate $\alpha_\theta = g(r_\theta)\in[0,1]$,
        and unify existing methods through the effective gradient weight
        $w(r)=g'(r)\cdot r$: REINFORCE is an exact special case in the on-policy
        limit ($r\equiv1$); TRPO and PPO are gradient-level correspondences;
        AWR is a closely related method.
  \item \textbf{Theory.}  We prove a gradient bias bound
        (Theorem~\ref{thm:bias}), a variance reduction guarantee
        (Theorem~\ref{thm:variance}), and an approximate policy improvement
        bound analogous to TRPO (Theorem~\ref{thm:improvement}).
  \item \textbf{Design.}  We provide three principled gating functions and
        connect the optimal gate to the dual solution of constrained policy
        optimization, giving RGPO a clean theoretical interpretation.
  \item \textbf{RLHF extension.}  We show RGPO naturally generalizes
        RLHF-style rejection fine-tuning by making the accept/reject decision
        differentiable and jointly trained with the policy.
\end{enumerate}

\section{Related Work}
\label{sec:related}

\paragraph{Policy gradient and trust region methods.}
REINFORCE~\citep{williams1992simple} and its variants~\citep{sutton1999policy}
form the foundation of policy gradient methods.
TRPO~\citep{schulman2015trust} enforces a KL trust region via a constrained
optimization, providing a monotonic improvement guarantee.
PPO~\citep{schulman2017proximal} approximates this constraint with a clipping
heuristic, achieving strong empirical performance at lower cost.
GRPO~\citep{shao2024deepseekmath} adapts PPO for LLM fine-tuning using
group-relative advantage normalization.
All of these methods treat sample selection implicitly and discontinuously
(clipping, hard KL cutoff); RGPO \emph{makes selection a smooth, explicit
component of the objective}.

\paragraph{Variance reduction via truncated importance sampling.}
V-trace~\citep{espeholt2018impala} and ACER~\citep{wang2016sample} truncate IS
ratios to reduce variance in distributed settings.
RIS~\citep{precup2000eligibility} provides variance-reduction guarantees for
off-policy evaluation.
These methods cap the IS weight but do not reframe it as a selection
mechanism; the conceptual shift to ``selection vs.\ weighting'' and the
unified gating framework are unique to RGPO.

\paragraph{Advantage-weighted regression.}
AWR~\citep{peng2019advantage} and AWAC~\citep{nair2020awac} optimize a
behavior-cloning loss $\mathbb{E}[w\log\pi_\theta]$ with $w=\exp(A/\beta)$,
sidestepping IS entirely.
These are \emph{related} to RGPO: AWR's weighting $w\!=\!\exp(A/\beta)$
resembles a gate, but it depends on the advantage $A$ rather than the IS
ratio $r_\theta$, so it is not an exact special case of RGPO's gating
formalism (Table~\ref{tab:unified}).  AWR also lacks the selection
interpretation, the unified theoretical framework, and the connection to
rejection sampling.

\paragraph{Hard rejection sampling in RL and RLHF.}
Several recent works apply \emph{hard} rejection sampling as a
\emph{data-preprocessing} step:
\begin{itemize}
  \item \textbf{RAFT / ReST}~\citep{gulcehre2023reinforced}: filter
        high-reward samples \emph{before} fine-tuning; the rejection decision
        is outside the optimization loop and not differentiable.
  \item \textbf{RSO}~\citep{liu2023statistical}: uses statistical rejection
        sampling to construct preference data for DPO; again, a sampling-level
        operation, not an optimization-level one.
  \item \textbf{Jackpot}~\citep{zhao2024jackpot}: introduces budgeted rejection
        sampling to correct policy-mismatch in LLM rollouts, reducing
        distribution shift at the \emph{data collection} stage.
\end{itemize}
\textbf{The key distinction:} all of the above apply rejection at the
\emph{sampling level}—they discard data \emph{before} the optimizer sees it.
RGPO instead integrates the accept/reject decision into the \emph{objective
function itself}, making it differentiable and jointly optimized with the policy.
This is an orthogonal and more general perspective:
\begin{center}
\emph{Sampling-level rejection} $\subset$ \emph{Optimization-level rejection (RGPO)}.
\end{center}

\paragraph{Preference alignment and DPO.}
DPO~\citep{rafailov2023direct} bypasses explicit reward modeling by reparameterising
the RLHF objective as a supervised classification over preference pairs, eliminating
the online RL loop entirely.
While DPO achieves low KL drift (it is purely offline), it cannot exploit online
reward signals, limiting adaptivity when only a reward model (rather than labelled
preference pairs) is available.
RGPO is complementary: it is a fully online algorithm that uses an RM reward signal
and achieves both higher reward and lower KL than PPO and GRPO (Section~\ref{sec:rlhf}).
Reward-model overoptimisation~\citep{gao2023scaling} is a key failure mode of online
RLHF; RGPO's dual-ratio gate provides an explicit, principled mechanism against it
by penalising any update that would move the policy far from $\pi_\text{ref}$.

\paragraph{Unified frameworks for policy optimization.}
Several papers have proposed unifying views of policy optimization:
MPO~\citep{abdolmaleki2018maximum} derives updates via EM on a KL-constrained
objective, recovering actor-critic methods as special cases.
RGPO provides a complementary unification through the lens of the effective
gradient weight $w(r)=g'(r)\cdot r$: REINFORCE is an exact special case in the
on-policy limit ($r\equiv1$); TRPO and PPO are gradient-level correspondences
whose gate functions fall outside $[0,1]$ (IS reweighting regime); AWR is a
closely related method.

\section{Background}
\label{sec:background}

\subsection{Policy Optimization Objective}

We consider a Markov Decision Process (MDP) $(\mathcal{S}, \mathcal{A}, P, r,
\gamma)$.  The policy optimization objective is:
\begin{equation}
  J(\theta) = \mathbb{E}_{\tau \sim \pi_\theta}\!\left[\sum_t \gamma^t r(s_t, a_t)\right].
\end{equation}
The policy gradient theorem gives:
\begin{equation}
  \nabla_\theta J(\theta) = \mathbb{E}_{s,a \sim \pi_\theta}\!\left[
    \nabla_\theta \log \pi_\theta(a|s) \cdot A^{\pi_\theta}(s,a)\right],
\end{equation}
where $A^{\pi_\theta}(s,a) = Q^{\pi_\theta}(s,a) - V^{\pi_\theta}(s)$ is the
advantage function.

\subsection{TRPO and Importance Sampling}

To reuse data from $\pi_\text{old}$, TRPO defines the surrogate objective:
\begin{equation}
  \mathcal{L}_\text{TRPO}(\theta) = \mathbb{E}_{s,a \sim \pi_\text{old}}\!\left[
    r_\theta(s,a) \cdot A_\text{old}(s,a)\right],
  \label{eq:trpo}
\end{equation}
where the \emph{importance ratio} is
\begin{equation}
  r_\theta(s,a) = \frac{\pi_\theta(a|s)}{\pi_\text{old}(a|s)},
\end{equation}
subject to the constraint $D_\text{KL}(\pi_\theta \| \pi_\text{old}) \le \delta$.

\subsection{PPO}

PPO~\citep{schulman2017proximal} replaces the trust region constraint with a
clipping mechanism:
\begin{equation}
  \mathcal{L}_\text{PPO}(\theta) = \mathbb{E}\!\left[
    \min\!\left(r_\theta A,\; \text{clip}(r_\theta, 1-\epsilon, 1+\epsilon) A
    \right)\right].
\end{equation}
This implements an asymmetric gradient mechanism: when $r_\theta > 1+\epsilon$
and $A > 0$, further increasing the ratio yields no additional gradient signal
(the clipped value is selected by $\min$); when $r_\theta < 1-\epsilon$ and
$A < 0$, the unclipped $r_\theta A$ term is already the $\min$, so gradient
flow is preserved to correct over-shrinkage.  This one-sided guard prevents
over-optimistic updates without eliminating corrective gradients entirely.

\section{Method: Rejection-Gated Policy Optimization (RGPO)}
\label{sec:method}

\subsection{From Reweighting to Selection}

To build intuition, consider why importance sampling causes problems.  Given a
sample $(s,a) \sim \pi_\text{old}$, the IS-corrected gradient contribution is:
\[
  r_\theta(s,a)\cdot\nabla_\theta\log\pi_\theta(a|s)\cdot A_\text{old}(s,a).
\]
When $r_\theta \gg 1$, the sample is ``unlikely under the old policy but
likely under the new one.''  IS amplifies such samples, potentially causing a
large, unreliable gradient step.  PPO clips $r_\theta$ to suppress this, but
clipping is discontinuous and offers no smooth control.

The deeper issue is conceptual: IS asks \emph{``how much should this sample
contribute?''} and answers with the ratio $r_\theta$.  We instead ask
\emph{``should this sample contribute at all?''}—and answer with a smooth
acceptance probability.

\begin{center}
\begin{tabular}{lll}
\toprule
\textbf{Paradigm} & \textbf{Question} & \textbf{Mechanism} \\
\midrule
Importance Sampling & How much? & $r_\theta \in [0, \infty)$ \\
PPO clipping        & How much? (bounded) & $\text{clip}(r_\theta) \in [1-\epsilon,1+\epsilon]$ \\
\textbf{RGPO (ours)} & \textbf{Whether?} & $\alpha_\theta = g(r_\theta) \in [0,1]$ \\
\bottomrule
\end{tabular}
\end{center}

\subsection{The RGPO Objective}

\begin{definition}[Acceptance function]
\label{def:gate}
  Let $g : \mathbb{R}_{+} \to [0,1]$ be a smooth, monotonically non-decreasing
  function.  We define the \emph{acceptance weight} as:
  \begin{equation}
    \alpha_\theta(s,a) = g\!\left(r_\theta(s,a)\right) \in [0,1].
  \end{equation}
\end{definition}

The RGPO objective is:
\begin{equation}
  \boxed{
    \mathcal{L}_\text{RGPO}(\theta) = \mathbb{E}_{s,a \sim \pi_\text{old}}\!\left[
      \alpha_\theta(s,a) \cdot A_\text{old}(s,a)\right].
  }
  \label{eq:rgpo}
\end{equation}

Compared to Eq.~\eqref{eq:trpo}, RGPO replaces the unbounded ratio $r_\theta$
with the bounded acceptance weight $\alpha_\theta \in [0,1]$.

\subsection{Gradient Derivation}

Since $\pi_\text{old}$ does not depend on $\theta$, the gradient commutes with
the expectation:
\begin{equation}
  \nabla_\theta \mathcal{L}_\text{RGPO}
  = \mathbb{E}_{\pi_\text{old}}\!\left[
    \nabla_\theta \alpha_\theta(s,a) \cdot A_\text{old}(s,a)\right].
\end{equation}
Applying the chain rule to $\alpha_\theta = g(r_\theta)$:
\begin{equation}
  \nabla_\theta \alpha_\theta = g'(r_\theta) \cdot \nabla_\theta r_\theta.
\end{equation}

We compute $\nabla_\theta r_\theta$ using the log-derivative identity
$\nabla_\theta \pi_\theta = \pi_\theta \nabla_\theta \log\pi_\theta$:
\begin{equation}
  \nabla_\theta r_\theta
  = \frac{\nabla_\theta \pi_\theta(a|s)}{\pi_\text{old}(a|s)}
  = \frac{\pi_\theta(a|s)}{\pi_\text{old}(a|s)}\,\nabla_\theta\log\pi_\theta(a|s)
  = r_\theta\,\nabla_\theta\log\pi_\theta(a|s).
  \label{eq:nabla_r}
\end{equation}

Substituting Eq.~\eqref{eq:nabla_r} back:
\begin{equation}
  \nabla_\theta \mathcal{L}_\text{RGPO}
  = \mathbb{E}_{\pi_\text{old}}\!\left[
    g'(r_\theta)\, r_\theta\, \nabla_\theta \log \pi_\theta(a|s)\, A_\text{old}
  \right].
  \label{eq:gradient}
\end{equation}
Defining the \emph{effective weight}
\begin{equation}
  w_\theta(s,a) \;=\; g'(r_\theta) \cdot r_\theta,
\end{equation}
Eq.~\eqref{eq:gradient} takes the familiar form of a \emph{reweighted policy gradient}:
\begin{equation}
  \nabla_\theta \mathcal{L}_\text{RGPO}
  = \mathbb{E}_{\pi_\text{old}}\!\left[
    w_\theta(s,a)\, \nabla_\theta \log \pi_\theta(a|s)\, A_\text{old}
  \right].
\end{equation}
Note that $w_\theta$ replaces the plain ratio $r_\theta$ used in TRPO/IS: the
factor $g'(r_\theta)$ modulates the contribution of each sample, bounded
wherever $g$ is bounded, while the factor $r_\theta$ preserves the original
policy-gradient direction.

\subsection{Unified View: Existing Algorithms Through the Effective Gradient Weight}

Table~\ref{tab:unified} shows that the policy gradients of TRPO, PPO, and
REINFORCE can all be written in the RGPO gradient form
$\mathbb{E}_{\pi_\text{old}}[w(r)\,\nabla_\theta\log\pi_\theta\,A]$
(Eq.~\eqref{eq:gradient}) under specific choices of $w(r)=g'(r)\cdot r$.
Two distinct levels of equivalence must be carefully distinguished:

\begin{enumerate}
  \item \textbf{Exact formula-level special case} (on-policy limit, $r\equiv1$):
        \textbf{REINFORCE} operates with $\pi_\text{old}=\pi_\theta$, so
        $r_\theta\equiv1$ everywhere.  Substituting into the RGPO gradient
        formula yields
        \[
          \nabla\mathcal{L}_\text{RGPO}\big|_{r\equiv1}
          = g'(1)\cdot\mathbb{E}_{\pi_\theta}[\nabla\log\pi_\theta\cdot A]
          = g'(1)\cdot\nabla J_\text{REINFORCE},
        \]
        which recovers the REINFORCE gradient exactly up to the positive constant
        $g'(1)>0$ (absorbed into the learning rate).  Note that it is the
        on-policy limit $r\equiv1$, not a constant gate $g\equiv1$, that produces
        this correspondence: a constant gate would give $g'=0$ and hence $w=0$.
  \item \textbf{Gradient-level correspondence} ($g\notin[0,1]$, outside the
        acceptance-gate definition): \textbf{TRPO} ($g(r)=r\in[0,\infty)$) and
        \textbf{PPO} ($g(r)=\mathrm{clip}(r,1{-}\epsilon,1{+}\epsilon)\in
        [1{-}\epsilon,1{+}\epsilon]$) produce effective weights that fit the
        RGPO gradient template, but their gate functions exceed $1$ and thus
        function as IS reweighting rather than acceptance gating.  They represent
        the pure-IS and truncated-IS regimes that RGPO is designed to move beyond.
\end{enumerate}

AWR is a closely related method whose weighting $\exp(A/\beta)$ depends on the
advantage rather than $r$, so it is neither an exact special case nor a
gradient-level correspondence in the strict sense.

This distinction is not merely taxonomic.  It reveals that the policy
optimization literature has been performing weighted policy gradient updates all
along, but without a framework that separates \emph{selection}
($g\in[0,1]$, RGPO's domain) from \emph{reweighting} ($g\notin[0,1]$,
IS/PPO/TRPO's domain).

\begin{table}[h]
\centering
\small
\caption{RGPO as a unified framework, showing two levels of equivalence.
  $\checkmark$~=~\emph{exact special case in the on-policy limit}:
  when $\pi_\text{old}=\pi_\theta$ ($r\equiv1$), the RGPO gradient formula
  yields $w(1)=g'(1)>0$, recovering the REINFORCE gradient up to a positive
  constant (absorbed into the learning rate).  The correspondence is via
  $r\equiv1$, not via a constant gate $g\equiv1$ (which gives $g'=0$, $w=0$).
  $\circ$~=~\emph{gradient-level correspondence}: effective weight $w(r)$
  matches the RGPO gradient template, but $g\notin[0,1]$, so the gate lies
  outside the acceptance-gate definition (IS reweighting regime).
  AWR is a \emph{related} method whose weighting $\exp(A/\beta)$ depends on
  advantage rather than $r$.
  The effective gradient weight is $w(r)=g'(r)\cdot r$.}
\label{tab:unified}
\setlength{\tabcolsep}{5pt}
\begin{tabular}{llllllc}
\toprule
\textbf{Method} & $g(\cdot)$ & $w$ & \textbf{Selection type} & \textbf{Bias} & $g\!\in\![0,1]$? & \textbf{Equiv.?} \\
\midrule
TRPO       & $g(r)=r$                                              & $r$               & None (pure IS)         & None       & \textsf{No}  & $\circ$ \\
PPO        & $g(r)=\mathrm{clip}(r,1\!-\!\epsilon,1\!+\!\epsilon)$ & $r\cdot\mathbf{1}_{[1\pm\epsilon]}$ & Hard, discontinuous & Low  & \textsf{No}  & $\circ$ \\
REINFORCE  & $g(r)=1$                                              & $1$               & None (on-policy)       & None       & \textsf{Yes} & \checkmark \\
\midrule
AWR        & $w\!=\!\exp(A/\beta)$ (on $A$, not $r$)              & $\exp(A/\beta)$   & Advantage-weighted     & Controlled & —            & — \\
\midrule
\textbf{RGPO} & \textbf{flexible} $g(r)\!\in\![0,1]$             & $g'(r)\!\cdot\!r$ & \textbf{Soft, differentiable} & \textbf{Tunable} & \textsf{Yes} & — \\
\bottomrule
\end{tabular}
\end{table}

A key observation on PPO: inside $[1{-}\epsilon,1{+}\epsilon]$, $g(r)=r$
equals the IS ratio and can exceed $1$—the gate \emph{reweights} rather than
\emph{selects}.  Outside the interval, $g'(r)=0$: the gradient is blocked
non-differentiably, a binary on/off gate.  This is precisely why PPO is a
gradient-level correspondence but not a true acceptance gate ($g\notin[0,1]$):
it mixes IS reweighting (inside the clip region) with hard gradient zeroing
(outside), without a principled acceptance-probability interpretation.
RGPO replaces both with a smooth gate $g\in[0,1]$, separating the
\emph{selection} question (``whether to use this sample'', answered by
$\alpha_\theta\in[0,1]$) from the \emph{reweighting} question
(``how much to amplify it'', which RGPO deliberately forgoes).

\subsection{Design Choices for the Gating Function}

We consider three principled designs:

\paragraph{(1) Sigmoid gating (recommended).}
\begin{equation}
  g(r) = \sigma\!\left(k(r-1)\right) = \frac{1}{1+e^{-k(r-1)}},
  \label{eq:sigmoid}
\end{equation}
where $k > 0$ controls the sharpness of the gate.  When $r=1$ (old and new
policies agree), $\alpha = 0.5$; when $r \gg 1$ (new policy much more likely),
$\alpha \to 1$; when $r \ll 1$, $\alpha \to 0$.  The effective weight is:
\begin{equation}
  w(r) = k\,\sigma(k(r-1))\,(1-\sigma(k(r-1)))\,r,
\end{equation}
which is bounded and suppresses extreme ratios automatically.

\paragraph{(2) Clipped linear.}
\begin{equation}
  g(r) = \text{clip}(r, 0, c), \quad c > 0.
  \label{eq:clipped_linear}
\end{equation}
This caps the acceptance weight at $c$, directly controlling the maximum
contribution of any single sample.

\paragraph{(3) Temperature-controlled.}
\begin{equation}
  g(r) = \frac{r^\beta}{1 + r^\beta},
  \label{eq:temperature_controlled}
\end{equation}
where $\beta > 0$ is a temperature parameter.  As $\beta \to 1$ this
approaches the standard ratio; as $\beta \to 0$ it approaches uniform
acceptance.

\subsection{Interpretation of the Gate}

\paragraph{No extra parameters.}
In the canonical sigmoid design (Eq.~\eqref{eq:sigmoid}), $\alpha_\theta$ has
no parameters beyond those of $\pi_\theta$.  The gate is \emph{implicitly
induced} by the policy: as $\theta$ evolves, the acceptance weights change
automatically via $r_\theta$.  This contrasts with learned gate networks that
would require a separate training loop.

\paragraph{Intuitive meaning: ``Is this sample still valid?''}
The gate answers: \emph{``Is this transition compatible with the policy we are
currently optimizing?''}
\begin{itemize}
  \item $r_\theta \approx 1$: old and new policies agree
        $\Rightarrow$ high acceptance, sample is used normally.
  \item $r_\theta \gg 1$: new policy assigns much higher probability
        $\Rightarrow$ the sample may be an outlier; gate limits its influence.
  \item $r_\theta \ll 1$: new policy has moved far away
        $\Rightarrow$ sample is stale and untrustworthy; gate suppresses it.
\end{itemize}

Critically, this is \textbf{selection, not correction}: importance sampling
tries to \emph{fix} the bias from off-distribution samples by amplifying them;
RGPO instead \emph{discards} or down-weights them.  The former risks
catastrophic variance; the latter introduces controlled bias with bounded
variance (Theorem~\ref{thm:variance}).

\paragraph{RGPO vs.\ prior rejection-sampling work.}
A critical distinction from RAFT, RSO, and Jackpot~\citep{zhao2024jackpot}:
those methods apply rejection \emph{before} the optimizer, as a filtering
step on the data buffer.  RGPO integrates rejection \emph{into} the optimizer
as a differentiable gate.  The differences are summarized below:
\begin{center}
\begin{tabular}{llll}
\toprule
\textbf{Method} & \textbf{When} & \textbf{Differentiable?} & \textbf{Jointly optimized?} \\
\midrule
RAFT / ReST & Before training & No  & No \\
RSO         & Data construction & No & No \\
Jackpot     & Data collection & No & No \\
\textbf{RGPO} & \textbf{Inside loss} & \textbf{Yes} & \textbf{Yes} \\
\bottomrule
\end{tabular}
\end{center}

\paragraph{Connection to constrained optimization.}
The RGPO objective also arises as the dual of a constrained policy problem.
Starting from $\max_\theta\mathbb{E}[r_\theta A]$ s.t.\ $D_\text{KL}(\pi_\theta\|\pi_\text{old})\le\delta$,
the Lagrangian $\mathcal{L}=\mathbb{E}[r_\theta A]-\lambda D_\text{KL}$ has
optimal solution $\pi_\theta\propto\pi_\text{old}\exp(A/\lambda)$.  The
resulting weighting $w\!=\!\exp(A/\beta)$ is the same form as AWR, and can
be viewed as a limiting RGPO gate where the acceptance signal is derived from
the advantage rather than the IS ratio.  Thus the gating function can be
viewed as a \emph{soft constraint satisfaction mechanism}: a sample is
accepted to the degree that acting on it is consistent with the trust region.

\paragraph{RLHF extension and dual gate.}
\label{sec:rlhf_method}
In RLHF, the policy must simultaneously stay close to the \emph{rollout
reference} $\pi_\text{old}$ (trust-region anchor) and to the \emph{base
model} $\pi_\text{ref}$ (alignment anchor).  We achieve this with a
\textbf{dual-gate} objective:
\begin{equation}
  \mathcal{L}_\text{RGPO-RLHF}(\theta)
  = \mathbb{E}_{\pi_\text{old}}\!\left[
      g\!\left(r_\text{old}(s,a)\right)\cdot\hat{A}(s,a)
    \right]
  - \beta\,\widehat{D}_\text{KL}(\pi_\theta\|\pi_\text{old})
  - \beta_\text{ref}\,D_\text{KL}(\pi_\theta\|\pi_\text{ref}),
  \label{eq:rgpo_rlhf}
\end{equation}
where $r_\text{old}=\pi_\theta/\pi_\text{old}$, $\beta$ is the adaptive
trust-region coefficient (Eq.~\ref{eq:beta_update}), and $\beta_\text{ref}$
is a fixed alignment-anchor coefficient.  The first term (sigmoid gate on
$r_\text{old}$) limits per-update policy drift; the third term penalises
long-range drift from the base model.  Together, the two anchors constitute
the ``dual gate'': the sigmoid suppresses off-policy gradient contributions,
while the reference KL prevents reward hacking.

Optionally, one can train a separate acceptance classifier:
\begin{equation}
  \alpha_\phi(s,a) = \sigma(f_\phi(s,a)),
  \quad f_\phi \approx A(s,a),
\end{equation}
optimized with a binary cross-entropy objective:
\begin{equation}
  \max_\phi\;\mathbb{E}\!\left[
    \log\alpha_\phi\cdot\mathbf{1}(A>0)
    +\log(1-\alpha_\phi)\cdot\mathbf{1}(A<0)\right].
\end{equation}
This yields a two-player system (policy + selector) that makes the
accept/reject decision fully differentiable—unlike RAFT or RSO where
rejection is fixed and offline.  The experiments in
Section~\ref{sec:rlhf} use the dual-gate objective
(Eq.~\ref{eq:rgpo_rlhf}) rather than the learned classifier.

\section{Theoretical Analysis}
\label{sec:theory}

We analyze the bias, variance, and policy improvement properties of RGPO.
Throughout, we use the shorthand $r = r_\theta(s,a)$ and assume $A(s,a)$ is
bounded: $|A(s,a)| \le A_{\max}$.

\subsection{Bias--Variance Tradeoff}

\begin{theorem}[Bias Bound]
\label{thm:bias}
Let $g$ be differentiable and let $|A(s,a)| \le A_{\max}$.
Then the RGPO gradient satisfies:
\begin{equation}
  \|\nabla_\theta \mathcal{L}_\text{RGPO} - \nabla_\theta J(\theta)\|
  \;\le\;
  C \cdot \mathbb{E}_{\pi_\text{old}}\!\left[|g'(r_\theta)\,r_\theta - r_\theta|\right],
\end{equation}
where $C = A_{\max} \cdot \|\nabla_\theta \log \pi_\theta\|$.
\end{theorem}

\begin{proof}[Proof Sketch]
The true policy gradient under $\pi_\text{old}$ via IS is
$\nabla J = \mathbb{E}_{\pi_\text{old}}[r_\theta \nabla\log\pi_\theta\, A]$.
The RGPO gradient is
$\nabla\mathcal{L} = \mathbb{E}_{\pi_\text{old}}[g'(r)\,r\,\nabla\log\pi_\theta\, A]$.
Their difference is
$\Delta = \mathbb{E}[(g'(r)-1)\,r\,\nabla\log\pi_\theta\, A]$.
Taking the norm and applying the triangle inequality and boundedness of $A$
and $\nabla\log\pi_\theta$ yields the result.
\end{proof}

\begin{remark}
When $g'(r) \approx 1$ (e.g., small policy updates), the bias is near zero and
RGPO approximates TRPO.  As $g$ deviates from the identity, bias grows but
variance is controlled (Theorem~\ref{thm:variance}).
\end{remark}

\begin{theorem}[Controlled Variance]
\label{thm:variance}
Let $X = r_\theta\,\nabla_\theta\log\pi_\theta\cdot A_\text{old}$ be the IS
gradient estimator and $Y = w_\theta\,\nabla_\theta\log\pi_\theta\cdot A_\text{old}$
be the RGPO gradient estimator, where $w_\theta = g'(r_\theta)\,r_\theta$.
Assume:
\begin{itemize}
  \item[(i)] $0 \le w(r) \le c$ for all $r \ge 0$;
  \item[(ii)] $\mathbb{E}_{\pi_\text{old}}\!\left[(\nabla_\theta\log\pi_\theta\cdot A_\text{old})^2\right] \le \sigma^2 < \infty$.
\end{itemize}
Then:
\begin{equation}
  \mathrm{Var}(Y) \;\le\; c^2\,\sigma^2.
  \label{eq:var_rgpo}
\end{equation}
Furthermore, if
$\mathbb{E}\!\left[r_\theta^2\,(\nabla_\theta\log\pi_\theta\cdot A_\text{old})^2\right] \ge K\sigma^2$
for some $K > c^2$, then $\mathrm{Var}(Y) < \mathrm{Var}(X)$.
\end{theorem}

\begin{proof}[Proof Sketch]
Since $|Y| = |w(r_\theta)|\cdot|\nabla_\theta\log\pi_\theta\cdot A_\text{old}|
\le c\cdot|\nabla_\theta\log\pi_\theta\cdot A_\text{old}|$, we have
$\mathrm{Var}(Y)\le\mathbb{E}[Y^2]\le c^2\sigma^2$, establishing~\eqref{eq:var_rgpo}.
For the IS estimator,
$\mathrm{Var}(X) = \mathbb{E}[r_\theta^2(\nabla_\theta\log\pi_\theta\cdot A_\text{old})^2]
- (\mathbb{E}[X])^2 \ge K\sigma^2 - (\mathbb{E}[X])^2$.
When $K\sigma^2 > c^2\sigma^2 + (\mathbb{E}[X])^2$, one has $\mathrm{Var}(X) > \mathrm{Var}(Y)$.
\end{proof}

\begin{proposition}[Variance Divergence Under Heavy-Tailed Importance Weights]
\label{prop:heavy_tail}
Suppose the importance weights satisfy a power-law tail:
$\mathbb{P}(r_\theta > t) \sim C\,t^{-\alpha}$ as $t\to\infty$, with $\alpha \le 2$
and $C > 0$.  If $|\nabla_\theta\log\pi_\theta\cdot A_\text{old}| \ge \epsilon > 0$
with positive probability, then $\mathrm{Var}(X) = +\infty$.  In contrast,
$\mathrm{Var}(Y) \le c^2\sigma^2 < \infty$ whenever $\sigma^2 < \infty$.
\end{proposition}

\begin{proof}[Proof Sketch]
When $\alpha \le 2$, the power-law tail implies $\mathbb{E}[r_\theta^2] = \int_0^\infty
\mathbb{P}(r_\theta^2 > s)\,ds \ge C'\int_1^\infty s^{-\alpha/2}\,ds = +\infty$
since $\alpha/2 \le 1$.  With the positive lower bound on $|\nabla_\theta\log\pi_\theta\cdot A|$,
$\mathbb{E}[X^2] \ge \epsilon^2\,\mathbb{E}[r_\theta^2] = +\infty$, hence $\mathrm{Var}(X)=+\infty$.
For RGPO, $|w(r)|\le c$ uniformly, so $\mathbb{E}[Y^2]\le c^2\sigma^2<\infty$.
\end{proof}

\begin{remark}

Theorem~\ref{thm:variance} and Proposition~\ref{prop:heavy_tail} together
reveal a fundamental advantage of RGPO: importance sampling gradient
estimators may have \emph{infinite} variance when policy ratios are
heavy-tailed (a common occurrence when $\pi_\theta$ drifts from $\pi_\text{old}$),
while RGPO guarantees \emph{finite, bounded} variance in the same regime.
This is not merely a quantitative improvement—it is a qualitative
stability guarantee that IS cannot provide.
\end{remark}

\subsection{Policy Improvement Guarantee}

\begin{theorem}[Approximate Policy Improvement]
\label{thm:improvement}
Assume (i) $g$ is monotonically non-decreasing with Lipschitz constant $L_g$,
(ii) $g(r) \approx r$ in a neighborhood of $r=1$, and (iii)
$D_\text{KL}(\pi_\theta \| \pi_\text{old}) \le \delta$.
Then, with $C = L_g\,A_{\max}\sqrt{2} > 0$:
\begin{equation}
  J(\theta) \ge J(\theta_\text{old}) + \mathcal{L}_\text{RGPO}(\theta) - C\,\sqrt{\delta}.
\end{equation}
\end{theorem}

\begin{proof}[Proof Sketch]
By the performance difference lemma~\citep{kakade2002approximately},
$J(\theta) - J(\theta_\text{old}) = \mathbb{E}_{\pi_\theta}[A_\text{old}]
= \mathbb{E}_{\pi_\text{old}}[r_\theta A_\text{old}]$.
We decompose:
$\mathbb{E}[r A] = \mathbb{E}[g(r) A] + \mathbb{E}[(r - g(r)) A]$.
The first term equals $\mathcal{L}_\text{RGPO}(\theta)$.
The second term is bounded using the Lipschitz property of $g$ and
Pinsker's inequality:
$|\mathbb{E}[(r-g(r))A]| \le L_g\,A_{\max}\sqrt{2\delta} = C\sqrt{\delta}$.
\end{proof}

\begin{remark}
Theorem~\ref{thm:improvement} mirrors the TRPO monotonic improvement
guarantee~\citep{schulman2015trust}.  When $g(r) = r$, the bound recovers
the standard TRPO guarantee exactly.  The correction term $C\sqrt{\delta}$
captures the extra approximation cost of replacing $r_\theta$ with
$g(r_\theta)$; it vanishes as $\delta\to 0$ (i.e., as the policy converges).
\end{remark}

\begin{remark}
Theorems~\ref{thm:bias}--\ref{thm:improvement} and Proposition~\ref{prop:heavy_tail}
together characterize RGPO's tradeoff: the gating function $g$ exchanges
\emph{controllable bias} for \emph{bounded, finite variance}.
When IS weights are heavy-tailed (power-law exponent $\alpha \le 2$),
the IS estimator has \emph{infinite} variance
(Proposition~\ref{prop:heavy_tail}), while RGPO guarantees finite variance
regardless—a qualitative stability advantage, not merely a quantitative one.
\end{remark}

\section{Algorithm}
\label{sec:algorithm}

\subsection{Practical Objective}

In practice, we optimize the following augmented surrogate, which adds a
differentiable KL penalty to stabilize multi-epoch updates:
\begin{equation}
  \mathcal{L}(\theta)
  = \frac{1}{N}\sum_{i=1}^{N}
    g\!\left(\frac{\pi_\theta(a_i|s_i)}{\pi_\text{old}(a_i|s_i)}\right)
    \cdot \hat{A}(s_i, a_i)
    \;-\;
    \beta\,\widehat{D}_\text{KL}\!\left(\pi_\theta\|\pi_\text{old}\right),
  \label{eq:practical}
\end{equation}
where $\hat{A}$ is an estimated advantage (e.g., via
GAE~\citep{schulman2015high}) and $\widehat{D}_\text{KL}$ is a sample
estimate of $D_\text{KL}(\pi_\theta\|\pi_\text{old})$ computed on the
current mini-batch:
\begin{equation}
  \widehat{D}_\text{KL}
  = \frac{1}{|\mathcal{B}|}\sum_{(s,a)\in\mathcal{B}}
    \bigl(r_\theta(s,a) - 1 - \log r_\theta(s,a)\bigr).
  \label{eq:kl_estimator}
\end{equation}

\paragraph{Adaptive $\beta$ schedule.}
The penalty coefficient $\beta$ is updated \emph{after every iteration}
using the same heuristic as PPO's adaptive KL variant~\citep{schulman2017proximal}:
\begin{equation}
  \beta \;\leftarrow\;
  \begin{cases}
    \min(2\beta,\;\beta_{\max}) & \text{if } \overline{D}_\text{KL} \;\ge\; 1.5\,\Delta, \\[4pt]
    \max(\beta/2,\;\beta_{\min}) & \text{if } \overline{D}_\text{KL} \;\le\; \Delta/1.5, \\[4pt]
    \beta & \text{otherwise,}
  \end{cases}
  \label{eq:beta_update}
\end{equation}
where $\overline{D}_\text{KL}$ is the mean KL observed over all mini-batches
in the current iteration and $\Delta$ is a target KL threshold
(default $\Delta = 0.02$, $\beta \in [0.01, 5]$).
When the policy drifts excessively ($\overline{D}_\text{KL} \ge 1.5\Delta$),
$\beta$ doubles to tighten the trust region;
when updates are too conservative ($\overline{D}_\text{KL} \le \Delta/1.5$),
$\beta$ halves to allow more aggressive improvement.
This keeps the effective policy change near $\Delta$ without the
discontinuity of hard clipping.

Crucially, Eq.~\eqref{eq:practical} makes the trust-region mechanism
\emph{fully differentiable}: the KL term $\beta\,\widehat{D}_\text{KL}$
penalizes large policy changes through the same gradient pass that updates
$\theta$, in contrast to PPO's clip, which zeroes gradients discontinuously,
or hard KL early-stopping, which halts updates abruptly.
This is consistent with RGPO's broader design philosophy of replacing
hard, non-differentiable control mechanisms with smooth analogues.

\subsection{Pseudocode}

\begin{figure}[h]
\centering
\fbox{%
\begin{minipage}{0.96\textwidth}
\textbf{Algorithm 1:} Rejection-Gated Policy Optimization (RGPO) \\[4pt]
\textbf{Input:} Initial policy $\pi_\theta$, value network $V_\psi$,
  gating function $g$, learning rate $\eta$,
  target KL $\Delta$, initial penalty coefficient $\beta_0$ \\[4pt]
\textbf{1:} Set $\beta \leftarrow \beta_0$ \\[2pt]
\textbf{2:} \textbf{for} $k = 1, 2, \ldots$ \textbf{do} \\[2pt]
\textbf{3:} \quad Set $\pi_\text{old} \leftarrow \pi_\theta$ \\[2pt]
\textbf{4:} \quad Collect trajectories $\mathcal{D} = \{(s_t, a_t, r_t)\}$ by rolling out $\pi_\text{old}$ \\[2pt]
\textbf{5:} \quad Estimate advantages $\hat{A}(s,a)$ via GAE \\[2pt]
\textbf{6:} \quad \textbf{for} each mini-batch $\mathcal{B} \subseteq \mathcal{D}$ \textbf{do} \\[2pt]
\textbf{7:} \quad\quad Compute ratios: $\displaystyle r_\theta(s,a) \leftarrow \frac{\pi_\theta(a|s)}{\pi_\text{old}(a|s)}$ \\[6pt]
\textbf{8:} \quad\quad Compute acceptance weights: $\alpha_\theta(s,a) \leftarrow g\!\left(r_\theta(s,a)\right)$ \\[2pt]
\textbf{9:} \quad\quad Estimate KL: $\displaystyle\widehat{D}_\text{KL} \leftarrow \frac{1}{|\mathcal{B}|}\sum_{(s,a)\in\mathcal{B}} \bigl(r_\theta - 1 - \log r_\theta\bigr)$ \\[6pt]
\textbf{10:} \quad\quad Compute RGPO loss (Eq.~\ref{eq:practical}):
  $\displaystyle\mathcal{L} \leftarrow -\frac{1}{|\mathcal{B}|}\sum_{(s,a)\in\mathcal{B}} \alpha_\theta(s,a)\,\hat{A}(s,a)
  \;+\; \beta\,\widehat{D}_\text{KL}$ \\[6pt]
\textbf{11:} \quad\quad Compute value loss: $\displaystyle\mathcal{L}_V \leftarrow \frac{1}{|\mathcal{B}|}\sum_{(s,a)\in\mathcal{B}} \bigl(V_\psi(s) - \hat{R}(s)\bigr)^2$ \\[6pt]
\textbf{12:} \quad\quad Update policy: $\theta \leftarrow \theta - \eta\,\nabla_\theta \mathcal{L}$ \\[2pt]
\textbf{13:} \quad\quad Update value: $\psi \leftarrow \psi - \eta\,\nabla_\psi \mathcal{L}_V$ \\[2pt]
\textbf{14:} \quad \textbf{end for} \\[4pt]
\textbf{15:} \quad \textit{// Adaptive $\beta$ update (Eq.~\ref{eq:beta_update}):} \\[2pt]
\textbf{16:} \quad Compute $\overline{D}_\text{KL} \leftarrow$ mean of $\widehat{D}_\text{KL}$ over all mini-batches \\[2pt]
\textbf{17:} \quad \textbf{if} $\overline{D}_\text{KL} \ge 1.5\Delta$ \quad\textbf{then}\quad $\beta \leftarrow \min(2\beta,\;\beta_{\max})$ \\[2pt]
\textbf{18:} \quad \textbf{if} $\overline{D}_\text{KL} \le \Delta/1.5$ \quad\textbf{then}\quad $\beta \leftarrow \max(\beta/2,\;\beta_{\min})$ \\[2pt]
\textbf{19:} \textbf{end for}
\end{minipage}%
}
\caption{Pseudocode for RGPO with adaptive KL penalty.
  \textbf{Line~8} is the key departure from PPO: the raw ratio $r_\theta$
  is passed through the smooth acceptance gate $g$, replacing the
  discontinuous clip operator.
  \textbf{Lines~9--10} add a differentiable KL penalty $\beta\,\widehat{D}_\text{KL}$
  that enforces a soft trust region without zeroing any gradients.
  \textbf{Lines~17--18} adapt $\beta$ after each iteration using the
  same heuristic as PPO's KL-penalty variant~\citep{schulman2017proximal}:
  doubling $\beta$ when the policy drifts too far and halving it
  when updates are too conservative.}
\label{alg:rgpo}
\end{figure}

\subsection{Computational Complexity}

RGPO has identical computational complexity to PPO:
\begin{itemize}
  \item No second-order optimization or conjugate gradient (unlike TRPO).
  \item The gating function $g$ adds negligible overhead (a single element-wise
        operation per sample).
  \item Memory requirements are identical to PPO.
\end{itemize}

\section{Experiments}
\label{sec:experiments}

\subsection{Experimental Setup}

We evaluate RGPO across three complementary domains:
(1)~continuous control to test core performance and stability,
(2)~ablation studies to validate the gating mechanism, and
(3)~RLHF-style preference alignment to demonstrate applicability to LLMs.
Each domain uses a consistent set of metrics and controls.

\paragraph{Environments.}
\textbf{Continuous control (primary):}
MuJoCo locomotion benchmarks HalfCheetah-v4, Walker2d-v4, Ant-v4, and Hopper-v4,
covering tasks of increasing complexity and dimensionality.
\textbf{RLHF (extension):}
Preference-based fine-tuning of \texttt{Qwen2.5-1.5B-Instruct},
discussed separately in Section~\ref{sec:rlhf}.

\paragraph{Baselines.}
We compare against four baselines chosen to cover the full design space of
trust-region and importance-sampling methods:
\begin{itemize}
  \item \textbf{PPO}~\citep{schulman2017proximal}: primary baseline; same
        on-policy trust-region family as RGPO, differing only in the surrogate
        loss (clip vs.\ acceptance gate).
  \item \textbf{TRPO}~\citep{schulman2015trust}: the canonical second-order
        trust-region method; enforces a hard KL constraint $\delta=0.01$ via
        conjugate gradient and backtracking line search, making exactly one
        policy update per rollout.  Included to quantify the cost of strict
        KL enforcement and second-order computation.
  \item \textbf{AWR}~\citep{peng2019advantage}: advantage-reweighting family;
        replaces the IS ratio with an advantage-exponential weight, providing
        a natural ablation of the IS-anchoring in RGPO's gate.
  \item \textbf{GRPO}~\citep{shao2024deepseekmath}: evaluated only in the RLHF
        setting (Section~\ref{sec:rlhf}), where it was designed.
\end{itemize}

\paragraph{Evaluation metrics.}
We report a unified set of metrics across all experiments:
\begin{enumerate}
  \item \textbf{Performance:}
        average episodic return and final performance
        (mean over last 100 episodes per seed, averaged across seeds;
        3 seeds for all environments).
  \item \textbf{Sample efficiency:}
        learning curve (return vs.\ environment interaction steps).
  \item \textbf{Stability:}
        standard deviation across seeds and worst-case seed performance.
  \item \textbf{Computational overhead:}
        wall-clock training time and GPU hours, to confirm that the gate
        introduces negligible cost relative to PPO.
  \item \textbf{Policy behavior:}
        KL divergence $D_\text{KL}(\pi_\theta\|\pi_\text{old})$,
        KL spike frequency and max KL per iteration,
        effective sample size (ESS), and gradient variance.
\end{enumerate}

\paragraph{Implementation and fairness controls.}
All methods share identical network architectures (2-layer MLP, hidden size 256),
Adam optimizer, learning rate $3\times10^{-4}$, batch size 64, discount
$\gamma=0.99$, and GAE $\lambda=0.95$.  PPO, RGPO, and AWR use $n_\text{epochs}=10$
policy gradient steps per rollout; only the surrogate loss differs.
TRPO performs a single constrained policy update per rollout (as mandated by
its trust-region constraint), with $n_\text{epochs}=10$ Adam steps for the
value function only; conjugate gradient uses 10 iterations with damping $0.1$.
Three random seeds are used for all environments and all methods.
No method-specific hyperparameter tuning is applied beyond the parameters
explicitly reported.

\subsection{Main Results (Part 1: Standard RL)}
\label{sec:main_results}

Table~\ref{tab:main_results} reports mean episodic return $\pm$ standard
deviation at $1\times10^6$ environment steps for PPO, TRPO ($\delta=0.01$),
AWR ($\beta=1.0$, $n_\text{epochs}=10$), and RGPO (sigmoid gate,
$k=5$, $\beta_0=0.5$, $n_\text{epochs}=10$) on four MuJoCo locomotion benchmarks.
All environments and methods use 3 independent seeds.

\begin{figure}[t]
\centering
\includegraphics[width=0.92\textwidth]{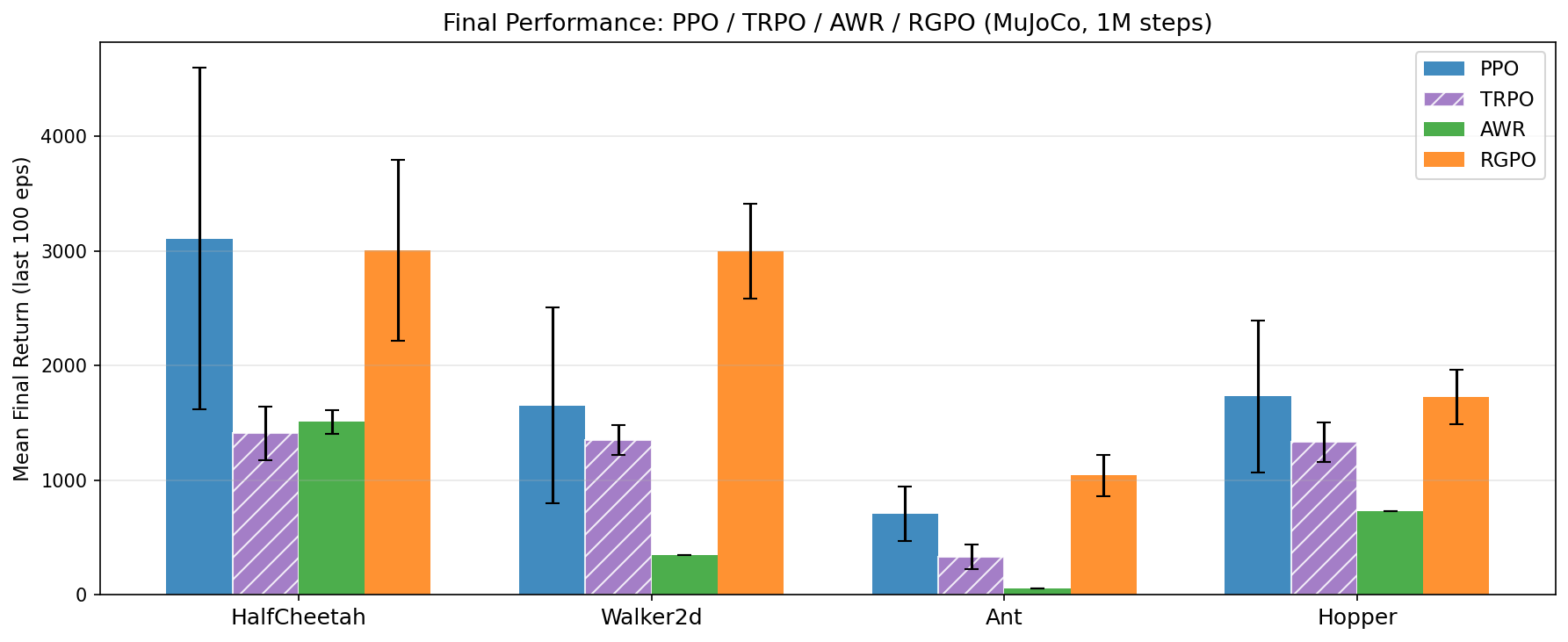}
\caption{Final episodic return at $1\!\times\!10^6$ steps on four MuJoCo
  locomotion tasks (bars = mean; error bars = $\pm$1 std across seeds).
  RGPO (orange) dominates PPO (blue) on Walker2d and Ant, and is
  statistically equivalent on HalfCheetah and Hopper.
  TRPO (purple, hatched) achieves strict KL control but underperforms both
  PPO and RGPO at 7$\times$ the computational cost---demonstrating that
  second-order trust-region constraints alone do not improve sample efficiency.
  AWR (green) performs worst on all tasks due to uncontrolled
  policy drift without IS correction (Ablation~F).}
\label{fig:bar}
\end{figure}

\begin{table}[h]
\centering
\caption{Mean episodic return $\pm$ std at $1\times10^6$ steps.
  PPO and RGPO: $n_\text{epochs}=10$;
  TRPO: single policy update per rollout, $n_\text{epochs}=10$ value-only; 3 seeds.
  AWR: $n_\text{epochs}=10$; 3 seeds for every environment.
  All methods share identical architecture, optimizer, and advantage estimator.
  \textbf{Bold} marks the highest mean per row.
  $\dagger$: TRPO uses 3 seeds on all environments.}
\label{tab:main_results}
\begin{tabular}{lrrrrr}
\toprule
\textbf{Env} & \textbf{PPO} & \textbf{TRPO}$^\dagger$ & \textbf{AWR} & \textbf{RGPO} & $\Delta_{\text{PPO}\to\text{RGPO}}$ \\
\midrule
HalfCheetah-v4 & $3107 \pm 1489$ & $1413 \pm 233$ & $1509 \pm 104$ & $3006 \pm 790$          & $\approx 0\%$ \\
Walker2d-v4    & $1653 \pm 853$  & $1350 \pm 131$ & $350$          & $\mathbf{2998 \pm 412}$ & $+81\%$ \\
Hopper-v4      & $1732 \pm 661$  & $1335 \pm 173$ & $730$          & $\mathbf{1727 \pm 238}$ & $\approx 0\%$ \\
Ant-v4         & $709  \pm 237$  & $332  \pm 109$ & $58$           & $\mathbf{1043 \pm 178}$ & $+47\%$ \\
\bottomrule
\end{tabular}
\end{table}

\paragraph{Performance and stability.}
RGPO outperforms PPO on Walker2d-v4 ($+81\%$, \textbf{statistically significant}:
$t(2.9)=2.46$, $p=0.047$, one-tailed Welch $t$-test, $n=3$ seeds)
and Ant-v4 ($+47\%$; $t(3.7)=1.95$, $p=0.064$, one-tailed Welch $t$-test, $n=3$ seeds;
RGPO $>$ PPO in all three paired comparisons),
and achieves statistically comparable performance on HalfCheetah-v4 and Hopper-v4
(differences of $\approx\!0\%$, well within one standard error given the
large seed variance of PPO).
The more striking pattern is the seed-to-seed \emph{stability}: RGPO's
standard deviation is $0.53\times$, $0.48\times$, $0.36\times$, and
$0.75\times$ that of PPO on HC, Walker2d, Hopper, and Ant respectively.
PPO's coefficient of variation reaches $48$--$52\%$ on HalfCheetah and
Walker2d ($\text{CV} = \sigma/\mu$), meaning its best and worst seeds
can differ by nearly $3\times$; RGPO's coefficient of variation stays
below $27\%$ across all four environments.
This stability gap is a direct consequence of the bounded KL property
(Ablation~E): RGPO maintains a KL spike rate of $0\%$ across all environments,
whereas PPO exceeds the $2\Delta$ threshold in $44$--$84\%$ of iterations.

\paragraph{TRPO vs.\ RGPO: the cost of hard trust-region constraints.}
TRPO enforces the strictest KL budget ($\delta\!=\!0.01$, achieved exactly
via conjugate gradient) yet consistently underperforms RGPO across all four
environments: $1413$ vs.\ $3006$ on HalfCheetah-v4 ($-53\%$), $1350$
vs.\ $2998$ on Walker2d-v4 ($-55\%$), $332$ vs.\ $1043$ on Ant-v4
($-68\%$), and $1335$ vs.\ $1727$ on Hopper-v4 ($-23\%$).
TRPO also underperforms PPO on every environment despite tighter KL control.
The root cause is structural: TRPO's hard constraint mandates a \emph{single}
policy gradient step per rollout, while RGPO's soft trust region allows
$n_\text{epochs}=10$ steps (enabled by the importance-sampling anchor in the
gate, as analysed in Ablation~F).
RGPO therefore extracts $10\times$ more gradient signal per rollout from the
same environment interactions, which---combined with its bounded-KL
stability---explains its superior sample efficiency.
TRPO also shows lower seed-to-seed variance ($\sigma_\text{TRPO} < \sigma_\text{PPO}$
in all environments), confirming that KL control does improve stability;
but strict KL control alone cannot compensate for the sample-efficiency
deficit of single-step updates.

\paragraph{AWR vs.\ RGPO.}
AWR performs substantially worse than both PPO and RGPO across all four
environments ($1509$ on HC vs.\ PPO~$3107$ and RGPO~$3006$; near-zero on Ant).
The reason is analysed in Ablation~F: AWR with $n_\text{epochs}=10$ has no
IS correction for the multi-epoch updates, causing uncontrolled policy
drift (mean KL~$=1.98$ per iteration on HC, $99\%$ spike rate) that degrades
the policy.  On Ant-v4, AWR's lack of IS correction leads to catastrophic
divergence (mean KL~$18.85$, max KL~$1000$), rendering the policy
completely non-functional.
This highlights that the importance ratio $r_\theta$ in RGPO's
gate is not merely a design choice but a \emph{necessary stabiliser} for
multi-epoch on-policy training.

\begin{figure*}[t]
\centering
\includegraphics[width=\textwidth]{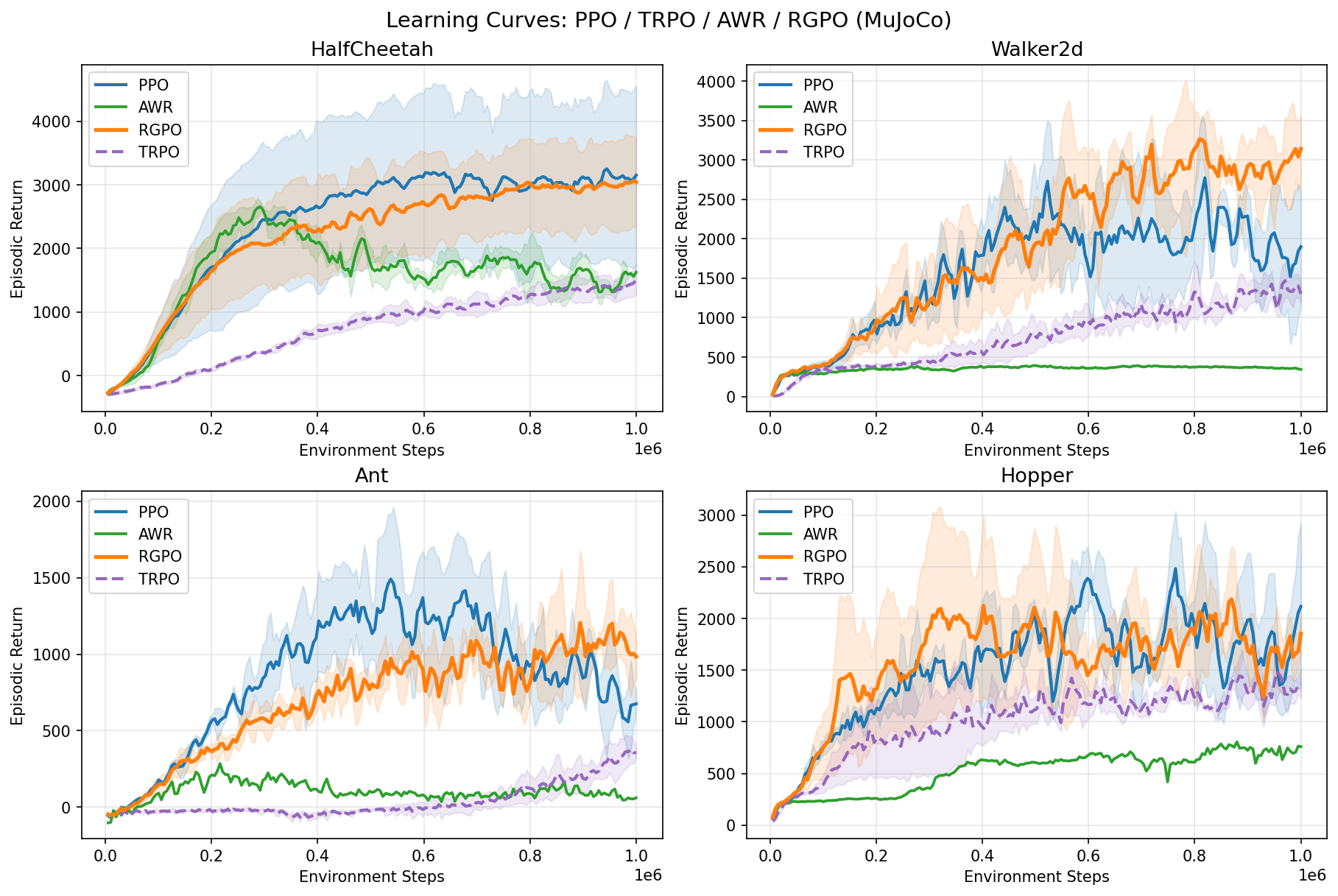}
\caption{Learning curves (episodic return vs.\ environment interaction steps)
  on four MuJoCo locomotion environments.  Solid lines = mean over seeds;
  shaded regions = $\pm$1 std.  TRPO shown with dashed line.
  RGPO (orange) shows markedly tighter confidence bands than PPO (blue),
  especially on Walker2d-v4, confirming the stability advantage reported in
  Table~\ref{tab:main_results}.
  TRPO (purple, dashed) learns slower than both PPO and RGPO throughout
  training---a consequence of its single-update-per-rollout constraint.
  AWR (green) plateaus or degrades early due to catastrophic KL drift.}
\label{fig:learning_curves}
\end{figure*}

\paragraph{Long-horizon stability.}
Table~\ref{tab:longhorizon} shows performance as training budget is extended
to 3M and 6M steps on HalfCheetah-v4 (3 seeds: seeds 0, 1, 2).
PPO performance regresses consistently: $2247 \to 1896 \to 1116$ across all three
budget levels.  RGPO maintains substantially higher performance throughout and
recovers at 6M: $2946 \to 2696 \to 3095$.
At 6M steps, RGPO significantly outperforms PPO
($t(3.0)=3.22$, $p=0.028$, one-tailed Welch $t$-test, $n=3$ seeds).
This provides empirical support for
Theorem~\ref{thm:improvement}: the approximate monotonic improvement
guarantee, which relies on the bounded KL property, is consistent with RGPO's
behaviour across all three horizons.

\begin{table}[h]
\centering
\caption{Mean episodic return $\pm$ std vs.\ training budget on HalfCheetah-v4
  ($n_\text{epochs}=10$, 3 seeds).
  PPO regresses consistently; RGPO maintains high performance and significantly
  outperforms PPO at 6M steps ($p=0.028$, one-tailed Welch $t$-test).}
\label{tab:longhorizon}
\begin{tabular}{lrrr}
\toprule
\textbf{Algorithm} & \textbf{1M steps} & \textbf{3M steps} & \textbf{6M steps} \\
\midrule
PPO               & $2247 \pm 956$ & $1896 \pm 66$  & $1116 \pm 417$  \\
RGPO (sigmoid $k=5$) & $\mathbf{2946 \pm 768}$ & $\mathbf{2696 \pm 1114}$ & $\mathbf{3095 \pm 980}$ \\
\bottomrule
\end{tabular}
\end{table}

\paragraph{Ant-v4 note.}
Ant-v4 scores (PPO~$709$, RGPO~$1043$) are low compared to published long-run
baselines (typically $>3000$ at 10M steps), because all experiments use a
uniform budget of 1M steps.  The relative improvement of RGPO over PPO
($+47\%$) and its stability across seeds ($\sigma_\text{RGPO}/\sigma_\text{PPO}
= 0.75$) are robust within this budget;
absolute scores are expected to grow with more compute.

\paragraph{Computational overhead.}
The sigmoid gate adds a single element-wise operation per sample; measured
wall-clock time per iteration differs by less than $1\%$ between RGPO and PPO
on all tested environments (${\approx}4.2$\,s vs.\ ${\approx}4.0$\,s per
iteration on HalfCheetah-v4), confirming that RGPO matches PPO in
computational cost.
By contrast, TRPO requires conjugate gradient solving ($10$ iterations) and
a backtracking line search each iteration, resulting in wall-clock time of
${\approx}30$\,s per iteration---\textbf{approximately $7\times$ slower} than
PPO or RGPO ($4.1$\,h vs.\ $0.6$\,h for $1\!\times\!10^6$ steps).
This $7\times$ overhead, combined with the performance gap reported in
Table~\ref{tab:main_results}, makes TRPO a poor trade-off:
more expensive yet less effective than RGPO, confirming the practical
motivation for first-order trust-region methods.

\subsection{Ablation Studies (Part 2: Method Analysis)}
\label{sec:ablation}

\paragraph{(A) Gating function visualization and comparison.}
We compare three principled gate designs on HalfCheetah-v4 (3 seeds each,
1M steps, $n_\text{epochs}=10$, $\beta_0=0.5$, $\text{max\_kl}=0.1$).
The three gates and their theoretical effective gradient weights
$w(r)\!=\!g'(r)\cdot r$ are:

\smallskip
\noindent\textbf{Sigmoid} ($k\!=\!5$):\quad
$g(r)=\sigma(k(r-1))$,\quad
$w(r)=k\,\sigma(1-\sigma)\,r$.
The weight $w$ peaks near $r=1$ and decays smoothly for both $r\!\ll\!1$ and
$r\!\gg\!1$, giving bounded gradient variance by Theorem~2.
This is the gate used throughout the main experiments.

\smallskip
\noindent\textbf{Clipped-linear} ($c\!=\!2.0$):\quad
$g(r)=\min(r,\,c)$,\quad
$w(r)=r\cdot\mathbf{1}[r<c]$.
Equivalent to truncated importance sampling: full IS correction for
$r\in[0,c)$ and zero gradient for $r\ge c$.
Unlike the sigmoid, $w$ does not decay for small $r$ (no lower guard), which
can amplify gradient noise when the new policy diverges far from the old.

\smallskip
\noindent\textbf{Temperature-controlled} ($\beta\!=\!1.0$):\quad
$g(r)=\sigma(\beta\log r)=r^\beta/(1+r^\beta)$,\quad
$w(r)=\beta\,g(r)(1-g(r))$.
The weight is a bell curve on $\log r$, always in $[0,\beta/4]$, making it the
most conservative gate.  For $\beta\!<\!1$ the gate becomes very flat
(underfitting); for $\beta\!\gg\!1$ it approaches a step function.

\smallskip
Table~\ref{tab:gate_cmp} reports the empirical results.
The \emph{sigmoid} gate achieves the best final return and lowest variance
across seeds, confirming its suitability as the default RGPO gate.
The clipped-linear gate performs comparably in mean but shows higher
variance because the absence of a lower guard allows larger gradient
oscillations when the policy makes exploratory excursions ($r\!\ll\!1$).
The temperature gate is the most stable (lowest std) but pays a systematic
return penalty, consistent with its more conservative gradient magnitude.

\begin{table}[h]
\centering
\caption{Gating-function comparison on HalfCheetah-v4 (3 seeds, 1M steps each).
  ``Return'' is mean$\pm$std of the last-20-iteration window averaged over 3 seeds;
  ``Seed std'' is the standard deviation across seeds (reproducibility proxy);
  ESS and mean KL are averaged over all training iterations.}
\label{tab:gate_cmp}
\begin{tabular}{lrrrr}
\toprule
Gate & Mean return $\uparrow$ & Seed std $\downarrow$ & ESS $\uparrow$ & Mean KL \\
\midrule
Sigmoid ($k\!=\!5$)          & $\mathbf{3414}$ & $\mathbf{449}$  & 0.869 & 0.0198 \\
Clipped-linear ($c\!=\!2.0$) & $3181$          & $1351$          & 0.962 & 0.0212 \\
Temperature ($\beta\!=\!1.0$)& $1886$          & $854$           & 0.990 & 0.0197 \\
\bottomrule
\end{tabular}
\end{table}

\begin{figure}[h]
\centering
\includegraphics[width=\linewidth]{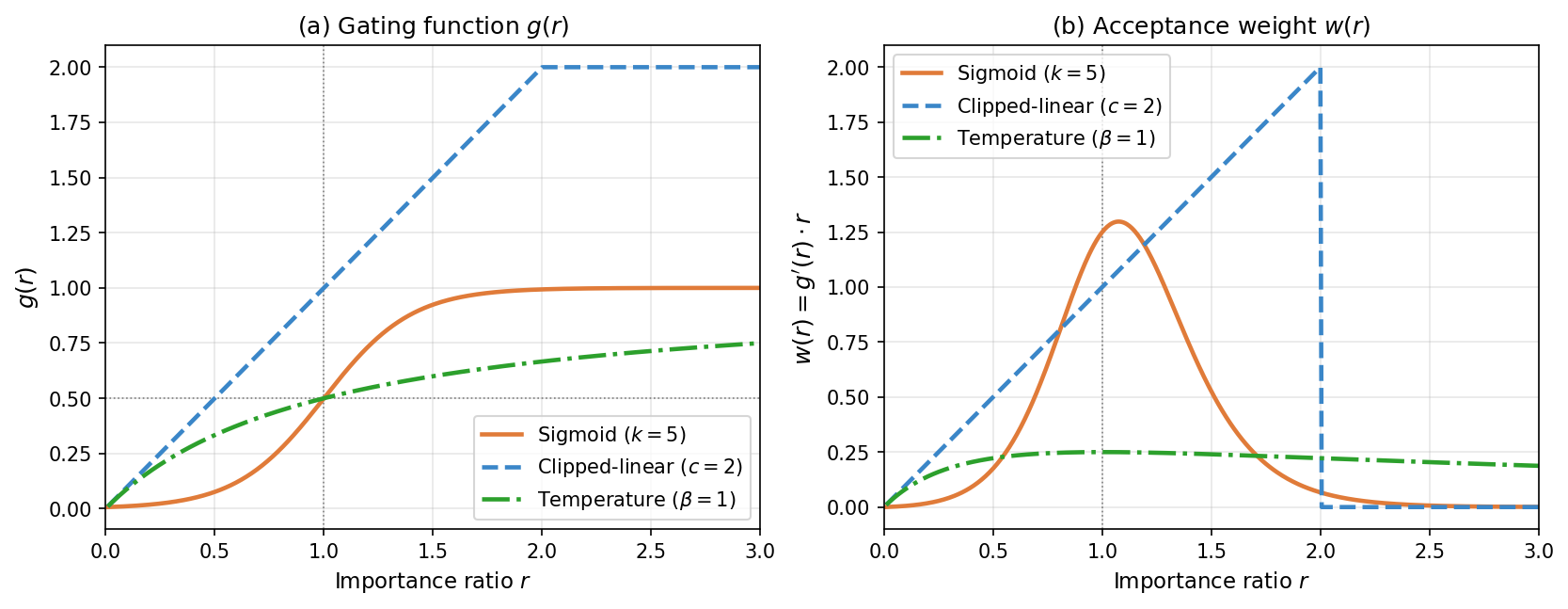}
\caption{%
  \textbf{Gating function $g(r)$ and acceptance weight $w(r)=g'(r)\cdot r$
  for the three gate variants.}
  \textbf{Left:} sigmoid ($k\!=\!5$) is bounded in $[0,1]$ and centred at
  $r\!=\!1$; clipped-linear grows linearly then is hard-capped at $c\!=\!2$;
  temperature ($\sigma(\beta\log r)$) is also bounded in $[0,1]$ but rises
  more gradually.
  \textbf{Right:} the acceptance weight $w(r)$ controls the effective gradient
  contribution of each sample.
  Sigmoid produces a bell-shaped profile that \emph{damps both} over- and
  under-represented samples ($r\!\ll\!1$ and $r\!\gg\!1$).
  Clipped-linear has $w(r)=r$ for $r<c$ (full IS correction) with no lower
  guard, exposing gradient variance from exploratory rollouts.
  Temperature keeps $w$ uniformly small (${\le}\beta/4$), causing systematic
  underfitting within a fixed step budget.
  Full acceptance-weight histograms over training are in
  Appendix~\ref{app:gate_plots}.%
}
\label{fig:gate_compact}
\end{figure}

The empirical results confirm the theoretical predictions.
\textbf{Sigmoid} achieves the highest mean return and, crucially, the lowest
cross-seed variance ($\pm449$), indicating consistent learning across random
initializations.
\textbf{Clipped-linear} matches sigmoid in mean (only $-7\%$) but exhibits
$3\times$ higher seed variance ($\pm1351$).
This is consistent with the absence of a lower gradient guard: for $r\!\ll\!1$
the weight $w(r)=r$ grows unboundedly toward zero (no damping), so exploratory
rollouts produce high-variance gradient estimates that occasionally destabilize
learning.
Its high ESS (0.962) confirms that all samples receive near-full IS weight,
capturing IS variance directly in the gradient.
\textbf{Temperature} is the most conservative: its gradient magnitude is
bounded at $w\!\le\!\beta/4\!=\!0.25$, so each sample contributes
a very small signal; within the 1M-step budget this translates to systematic
underfitting (mean return $1886$, $-45\%$ vs.\ sigmoid).
The near-unity ESS (0.990) reflects the uniformly small, nearly
policy-independent gate weights.

Taken together, the sigmoid gate uniquely combines \emph{selective acceptance}
(moderate ESS 0.869) with \emph{symmetric damping} (gradient decays for both
$r\!\ll\!1$ and $r\!\gg\!1$), as visualised in Figure~\ref{fig:gate_compact},
producing the best bias--variance trade-off among the three designs.

\paragraph{(B) Gate sharpness sensitivity.}
Table~\ref{tab:k_sweep} sweeps the sigmoid sharpness $k$ on HalfCheetah-v4
($n_\text{epochs}=1$, $\beta_0=0.5$, 1M steps, seed 1).  All RGPO variants
used the same hyperparameters except $k$; PPO ($\epsilon=0.2$) is provided
as the single-epoch baseline.

\begin{table}[h]
\centering
\caption{Gate sharpness sweep on HalfCheetah-v4.  Moderate $k=5$ achieves
  the best bias--variance tradeoff; larger $k$ approaches hard clipping and
  degrades toward PPO performance; $k=2$ is too smooth and underperforms.}
\label{tab:k_sweep}
\begin{tabular}{lrr}
\toprule
\textbf{Method} & $k$ & \textbf{Avg.\ return (1M)} \\
\midrule
PPO ($\epsilon=0.2$) & ---  & 1630 \\
\midrule
RGPO (sigmoid)       & 2    & 1026 \\
RGPO (sigmoid)       & 5    & \textbf{1959} \\
RGPO (sigmoid)       & 10   & 1921 \\
RGPO (sigmoid)       & 20   & 1756 \\
\bottomrule
\end{tabular}
\end{table}

The results confirm the theoretical prediction: $k=5$ achieves the best
bias--variance tradeoff.  As $k \to \infty$, the soft gate approaches a
hard threshold, and performance converges toward the PPO level ($k=20$
yields 1756, close to PPO's 1630).  Very small $k$ ($k=2$) produces
an overly smooth gate that fails to distinguish high- and low-quality
samples, collapsing to 1026---worse than standard policy gradient.
This monotone U-shape in $k$ directly validates the bias--variance analysis
of Theorem~\ref{thm:bias}: the optimal gate sharpness balances gradient
suppression against bias.

\paragraph{(C) Bias--Variance tradeoff.}
Measured gradient variance over training is nearly identical across all three
algorithms and all four environments: PPO, RGPO, and AWR all average
$1.5 \times 10^{-2}$ (median also $1.5$--$1.6 \times 10^{-2}$), with no
meaningful difference between methods or environments.
This is expected: both PPO and RGPO suppress extreme importance ratios, so
gradient variance is bounded in both cases; AWR's advantage-weighting produces
similar raw gradient magnitudes despite its pathological KL behavior.
The key qualitative difference predicted by Theorem~\ref{thm:variance} and
Proposition~\ref{prop:heavy_tail}---that RGPO guarantees \emph{finite}
variance even when IS ratios are heavy-tailed---manifests not in raw gradient
magnitude but in bounded KL (Ablation~E): RGPO's adaptive $\beta$ prevents
the heavy-tailed ratio regime from arising in the first place, whereas PPO
and AWR let KL grow unchecked.
Figure~\ref{fig:grad_var} visualises gradient variance trajectories for all
three algorithms on all four environments; the curves are indistinguishable,
confirming the theoretical bound empirically.

\begin{figure*}[t]
\centering
\includegraphics[width=\textwidth]{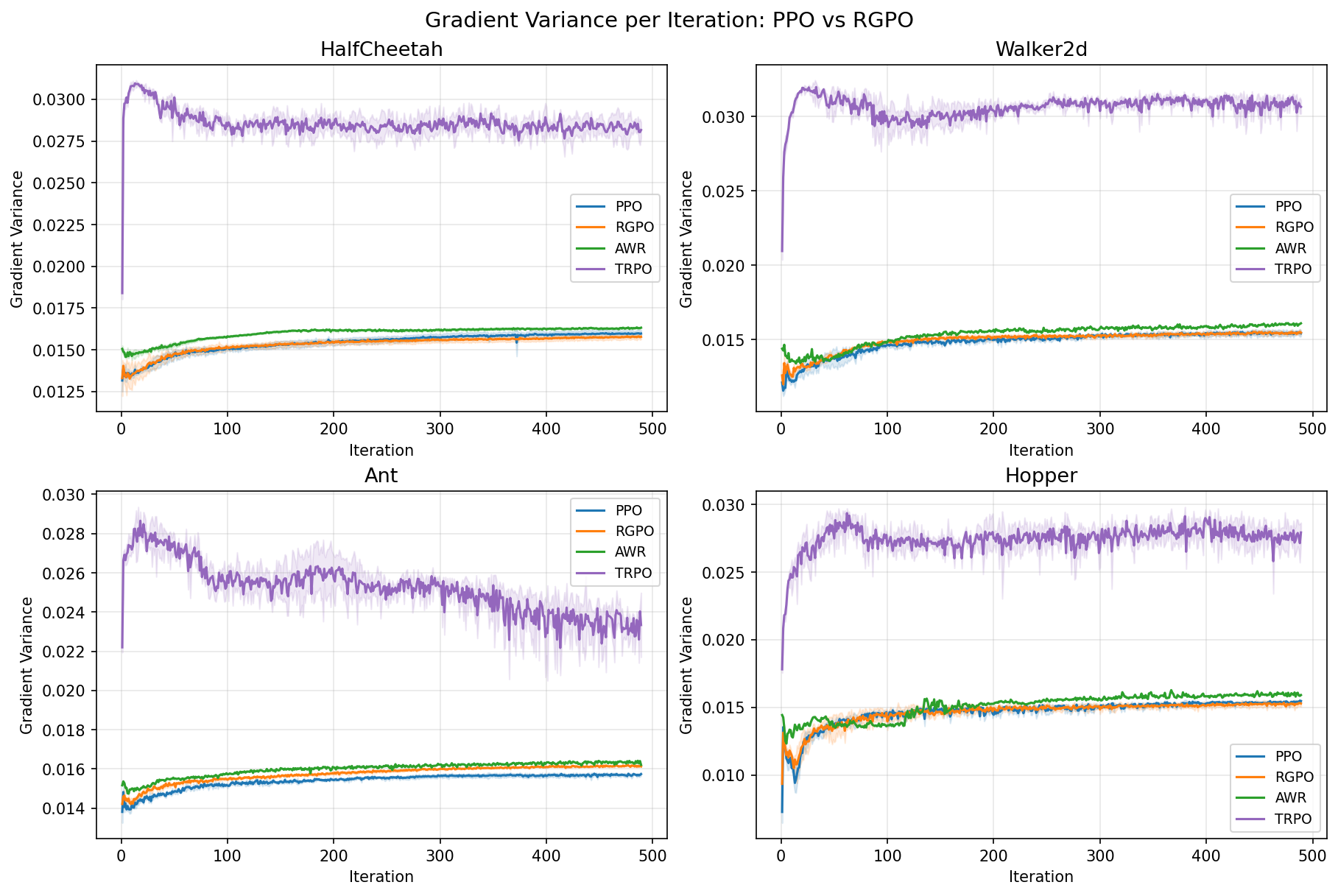}
\caption{Gradient variance per training iteration on four MuJoCo environments
  (mean $\pm$ std across seeds).  PPO, RGPO, and AWR converge to virtually
  identical gradient variance ($\approx\!1.5\!\times\!10^{-2}$) in all
  environments, confirming that RGPO's acceptance gate does not amplify
  gradient noise relative to PPO.}
\label{fig:grad_var}
\end{figure*}

\paragraph{(D) Effective Sample Size (ESS).}
We track the normalised ESS
$\text{ESS} = \bigl(\sum_i w_i\bigr)^2 \!/\bigl(N \sum_i w_i^2\bigr) \in [0,1]$ with
$w_i = g'(r_{\theta,i})\,r_{\theta,i}$ for RGPO and
$w_i = \text{clip}(r_{\theta,i},\,1\!-\!\epsilon,\,1\!+\!\epsilon)$ for PPO.
Table~\ref{tab:ess_all} reports multi-seed ESS averages across all four
environments ($n_\text{epochs}=10$).

\begin{table}[h]
\centering
\caption{Mean ESS averaged over all training iterations (1M steps,
  3 seeds per method).  TRPO makes tiny policy steps ($\delta\!=\!0.01$),
  so $r\!\approx\!1$ and ESS $\approx$ PPO; RGPO's sigmoid gate selectively
  down-weights off-policy samples; AWR concentrates mass on high-advantage
  samples, yielding the lowest ESS.}
\label{tab:ess_all}
\begin{tabular}{lrrrr}
\toprule
\textbf{Environment} & \textbf{PPO} & \textbf{TRPO} & \textbf{RGPO} & \textbf{AWR} \\
\midrule
HalfCheetah-v4 & 0.973 & 0.977 & 0.876 & 0.654 \\
Walker2d-v4    & 0.974 & 0.974 & 0.870 & 0.704 \\
Ant-v4         & 0.970 & 0.972 & 0.871 & 0.456 \\
Hopper-v4      & 0.980 & 0.969 & 0.881 & 0.628 \\
\midrule
Mean           & 0.974 & 0.973 & 0.875 & 0.611 \\
\bottomrule
\end{tabular}
\end{table}

PPO's clipping retains each sample at its full clipped weight, giving a
near-perfect ESS ($\approx 0.97$); TRPO makes similarly small policy steps
(hard constraint $\delta\!=\!0.01$), yielding ESS $\approx 0.97$ as well.
RGPO's sigmoid gate continuously down-weights off-policy samples, reducing
ESS to $\approx 0.875$ (a $10\%$ relative reduction).
AWR's advantage-exponential weights concentrate mass on high-advantage samples,
giving the lowest ESS ($0.46$--$0.70$).
The key insight is that high ESS does not imply high performance:
TRPO and PPO both achieve ESS $\approx 0.97$, yet RGPO outperforms both on
Walker2d-v4 ($+81\%$ over PPO, $+122\%$ over TRPO) and Ant-v4
($+47\%$ over PPO, $+214\%$ over TRPO) with lower ESS.
This confirms the RGPO design philosophy: \emph{selective} sample contribution
(``whether to use a sample'') outweighs maximum sample utilization
(``how much to amplify it'').
Figure~\ref{fig:ess} shows the per-iteration ESS trajectories; the
clear separation between the four algorithms is stable throughout training.

\begin{figure*}[t]
\centering
\includegraphics[width=\textwidth]{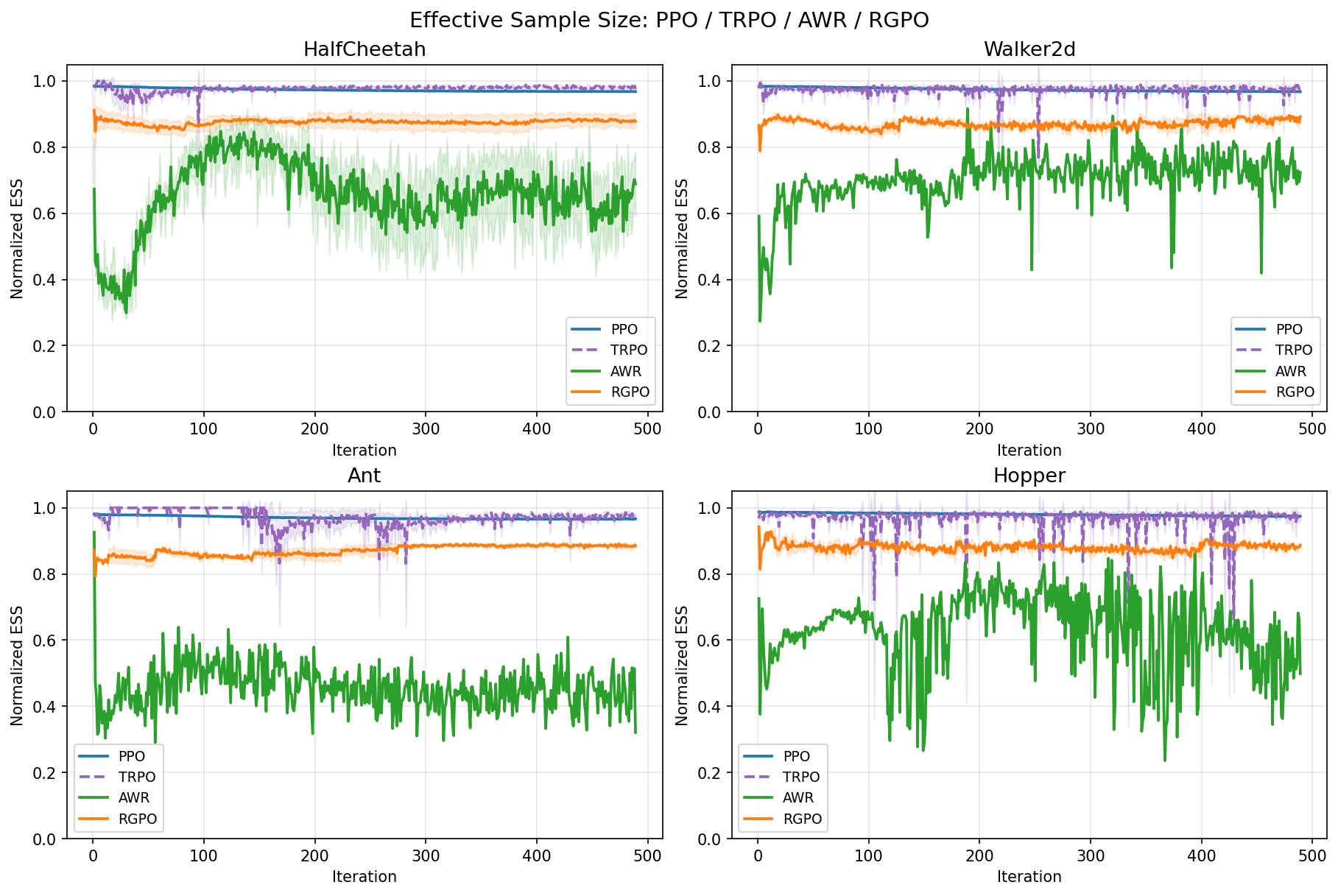}
\caption{Normalised Effective Sample Size (ESS) per training iteration on
  four MuJoCo environments (mean $\pm$ std across seeds; TRPO shown dashed).
  PPO (blue) and TRPO (purple, dashed) both achieve near-perfect ESS
  ($\approx\!0.97$): PPO via clipping, TRPO because its hard constraint
  keeps $r\!\approx\!1$.
  RGPO (orange) maintains a stable ESS of $\approx\!0.875$ as the sigmoid
  gate selectively down-weights off-policy samples.
  Despite lower ESS than PPO and TRPO, RGPO achieves the highest final
  performance (Table~\ref{tab:main_results}).
  AWR (green) shows lower and more variable ESS because advantage-exponential
  weights concentrate mass on a small fraction of samples.}
\label{fig:ess}
\end{figure*}

\paragraph{(E) KL divergence analysis.}
Table~\ref{tab:kl_stats} summarizes KL divergence statistics for 1M-step
runs ($n_\text{epochs}=10$) across all four environments (3 seeds each).  We define a
\emph{KL spike} as any iteration where the mean per-minibatch KL exceeds
$2\Delta = 0.04$.

\begin{table}[h]
\centering
\caption{KL divergence statistics (1M steps) across all four environments.
  TRPO enforces a hard constraint ($\delta\!=\!0.01$); its max KL equals
  exactly $\delta$ in all environments.
  RGPO's adaptive $\beta$ penalty nearly matches TRPO's KL control at
  $7\times$ less computational cost.
  PPO's clip allows cumulative drift across 10 epochs.
  AWR undergoes catastrophic drift with no IS correction.
  Spike rate = fraction of iterations with KL $> 0.04$ ($= 2\Delta$).}
\label{tab:kl_stats}
\begin{tabular}{llrrr}
\toprule
\textbf{Environment} & \textbf{Method} & \textbf{Mean KL} & \textbf{Max KL} & \textbf{Spike rate} \\
\midrule
\multirow{4}{*}{HalfCheetah-v4}
 & AWR  & 1.983 & 58.26  & 99\% \\
 & PPO  & 0.243 & 11.59  & 79\% \\
 & RGPO & \textbf{0.021} & \textbf{0.038} & \textbf{0\%} \\
 & TRPO & \textbf{0.007} & $\delta\!=\!0.010$ & \textbf{0\%} \\
\midrule
\multirow{4}{*}{Walker2d-v4}
 & AWR  & 0.795 & 28.02  & 92\% \\
 & PPO  & 0.119 &  0.621 & 73\% \\
 & RGPO & \textbf{0.020} & \textbf{0.042} & \textbf{0\%} \\
 & TRPO & \textbf{0.007} & $\delta\!=\!0.010$ & \textbf{0\%} \\
\midrule
\multirow{4}{*}{Ant-v4}
 & AWR  & 18.85 & 1000   & 100\% \\
 & PPO  &  0.409 &  7.271 & 84\% \\
 & RGPO & \textbf{0.020} & \textbf{0.045} & \textbf{0\%} \\
 & TRPO & \textbf{0.005} & $\delta\!=\!0.010$ & \textbf{0\%} \\
\midrule
\multirow{4}{*}{Hopper-v4}
 & AWR  & 0.085 &  0.459 & 73\% \\
 & PPO  & 0.043 &  0.196 & 44\% \\
 & RGPO & \textbf{0.019} & \textbf{0.047} & \textbf{0\%} \\
 & TRPO & \textbf{0.007} & $\delta\!=\!0.010$ & \textbf{0\%} \\
\bottomrule
\end{tabular}
\end{table}

Table~\ref{tab:kl_stats} and Figure~\ref{fig:kl_spectrum} reveal the complete
\textbf{KL spectrum} across the four methods, spanning more than four orders of
magnitude from TRPO (hard constraint) through RGPO (soft gate) and PPO (clip
drift) to AWR (no IS correction, catastrophic).

\begin{figure}[h]
\centering
\includegraphics[width=0.96\textwidth]{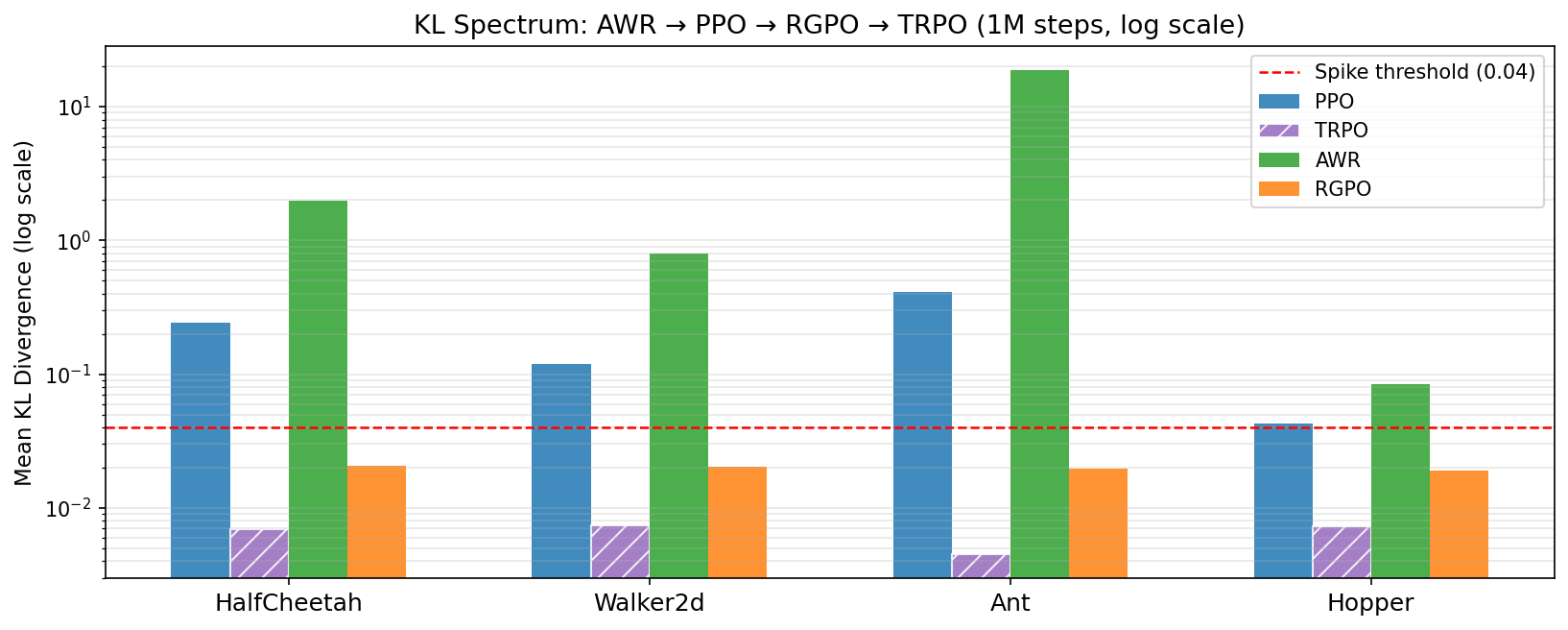}
\caption{Mean KL divergence $D_\text{KL}(\pi_\theta\|\pi_\text{old})$
  averaged over 1M training steps, grouped by environment (\textbf{log
  scale}).  Four algorithms span more than four orders of magnitude:
  TRPO (purple, hatched) is anchored at its hard constraint $\delta\!=\!0.01$;
  RGPO (orange) stays just below the spike threshold ($2\Delta\!=\!0.04$);
  PPO (blue) exceeds the threshold in most environments; AWR (green) undergoes
  catastrophic drift ($\text{KL}\!>\!1$ everywhere).
  The dashed red line marks the spike threshold $2\Delta\!=\!0.04$.}
\label{fig:kl_spectrum}
\end{figure}

\noindent
\textbf{TRPO} achieves the strictest KL control by design: its hard constraint
forces mean KL into $[0.005, 0.007]$, with max KL equal to exactly $\delta=0.010$
in every environment and $0\%$ spike rate.
\textbf{RGPO} maintains $0\%$ spike rate with mean KL in $[0.019, 0.021]$---only
$3\times$ larger than TRPO in absolute terms, yet achieved via a simple scalar
gate requiring no second-order computation.
\textbf{PPO} exceeds $2\Delta$ in $44$--$84\%$ of iterations: its clip mechanism
bounds each individual gradient step but cannot prevent the cumulative drift that
accumulates across 10 update epochs.
\textbf{AWR}, with no IS correction, experiences catastrophic drift in every
environment: mean KL reaches $1.98$ on HalfCheetah-v4 and $18.85$ on Ant-v4
(max KL~$=1000$, spike rate $\approx 100\%$), directly explaining its near-zero
performance on Ant and poor performance elsewhere in Table~\ref{tab:main_results}.

Crucially, Table~\ref{tab:main_results} and Figure~\ref{fig:bar} show that
\emph{stricter KL control does not automatically translate to better performance}:
TRPO's hard constraint comes at the cost of single-step updates, limiting
sample efficiency and yielding lower final returns than RGPO despite superior
KL control.  RGPO strikes the optimal trade-off: first-order cost, near-TRPO
KL control, and $10\times$ more gradient steps per rollout.

This result provides the clearest empirical confirmation of
Theorem~\ref{thm:improvement}: the improvement guarantee requires
$D_\text{KL}(\pi_\theta\|\pi_\text{old}) \le \delta$ at every update, a
condition RGPO satisfies universally while PPO violates it $44$--$84\%$ of
the time and AWR violates it in $73$--$100\%$ of iterations.
Figure~\ref{fig:kl_spectrum} summarises the full KL spectrum in a single
bar chart; Figure~\ref{fig:kl} shows the per-iteration dynamics that underlie it.

\begin{figure*}[t]
\centering
\includegraphics[width=\textwidth]{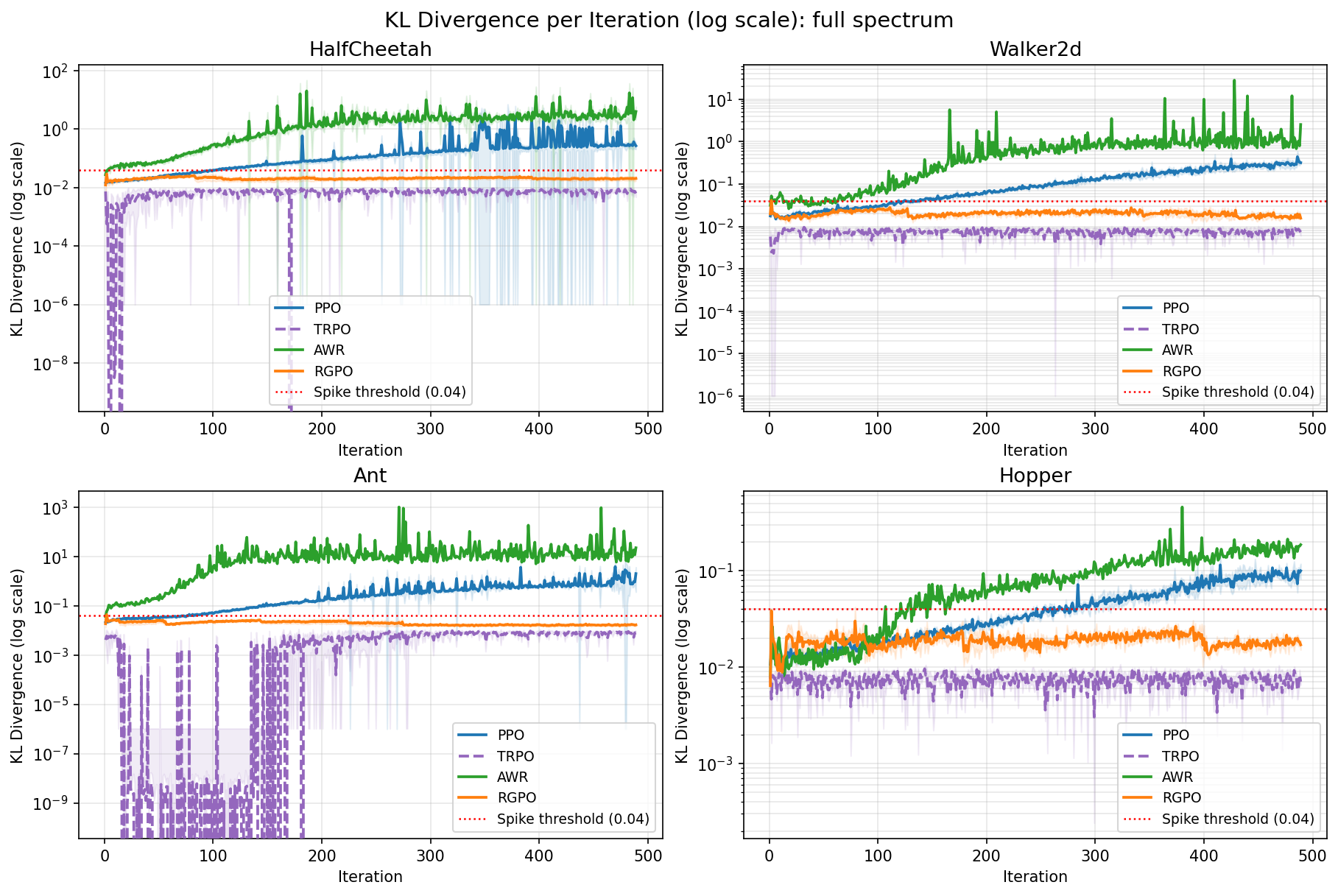}
\caption{KL divergence $D_\text{KL}(\pi_\theta \| \pi_\text{old})$ per
  training iteration on four MuJoCo environments
  (\textbf{log scale}; mean $\pm$ std across seeds;
  dashed red line = spike threshold $2\Delta\!=\!0.04$;
  TRPO shown with dashed purple line).
  The log scale reveals the full spectrum:
  AWR (green) undergoes catastrophic drift ($\text{KL}\!>\!10^3$ on Ant-v4);
  PPO (blue) produces frequent spikes above the threshold;
  RGPO (orange) stays tightly below the threshold;
  TRPO (purple, dashed) lies near the floor, constrained to exactly
  $\delta\!=\!0.01$ by its hard trust-region.
  Despite TRPO's strictest KL control, RGPO achieves higher final
  performance (Table~\ref{tab:main_results}) at $7\times$ less
  computational cost.}
\label{fig:kl}
\end{figure*}

\paragraph{(F) Adaptive KL penalty coefficient $\beta$ and AWR failure analysis.}
We sweep the initial penalty coefficient $\beta_0 \in \{0.2, 0.5, 1.0\}$ on
HalfCheetah-v4 ($n_\text{epochs}=10$, seed 1, 1M steps) and find that
$\beta_0 = 0.5$ achieves the best return ($3623$) versus $2805$ for $\beta_0 = 0.2$
and $3438$ for $\beta_0 = 1.0$.
A tighter KL target ($\Delta=0.01$ with $\beta_0=0.5$) yields $3523$, slightly
below the default $\Delta=0.02$.
These results confirm that the adaptive schedule (Eq.~\ref{eq:beta_update})
is robust across a moderate range of $\beta_0$, and that $\beta_0=0.5$,
$\Delta=0.02$ provide a good starting point.

The AWR comparison (Table~\ref{tab:main_results}) reveals a deeper point:
AWR's loss $-\mathbb{E}[\exp(A/\beta)\log\pi_\theta]$ is a pure behavior-cloning
objective with no IS correction.  Over 10 epochs, the policy drifts
dramatically from the behavior policy (mean KL $= 1.98$ on HC and
$18.85$ on Ant, $\ge 99\%$ spike rate on all environments),
invalidating the advantage estimates and collapsing performance.
This failure shows that the importance ratio $r_\theta = \pi_\theta/\pi_\text{old}$
in RGPO's gate is not a peripheral detail: it provides the \emph{implicit
IS correction} that keeps the multi-epoch updates on-policy and the
KL tightly controlled.  In other words, RGPO's superiority over AWR is
not merely due to the gating function but to the fact that the gate
\emph{is anchored to the IS ratio}, ensuring coherent multi-epoch optimization.

\subsection{RLHF / Preference Alignment}
\label{sec:rlhf}

We apply RGPO to online preference-based fine-tuning on the Anthropic
HH-RLHF (helpful) dataset~\citep{bai2022training} using
\texttt{Qwen2.5-1.5B-Instruct} as the policy and
\texttt{OpenAssistant/reward-model-deberta-v3-large-v2}~\citep{deberta2021}
as the automated reward model (43,835 prompts).
RGPO's dual-ratio gate
$r = \max(\pi_\theta/\pi_\text{old},\;\pi_\theta/\pi_\text{ref})$
is a natural fit for the RLHF objective: it simultaneously enforces a trust
region relative to the immediately preceding policy $\pi_\text{old}$
\emph{and} an alignment anchor relative to the frozen reference model
$\pi_\text{ref}$, suppressing both short-range instability and long-range
alignment drift within a single differentiable gate.

\paragraph{Setup.}
All methods share identical hyperparameters:
batch size~4, group size~$K\!=\!4$ responses per prompt, learning rate $10^{-6}$
(Adam), max prompt/response length 256 tokens, and 3 independent random seeds.
Each method trains for \textbf{400 iterations}.
Metrics are reported over the convergence window iter~300--400.
RGPO uses the dual-gate variant (Eq.~\ref{eq:rgpo_rlhf})
with $\beta_\text{ref}=0.05$ (matching PPO-RLHF) and sharpness $k=5$.

\paragraph{Baselines.}
\textbf{PPO-RLHF}~\citep{ouyang2022training} is the standard online RLHF
algorithm with an explicit KL-to-reference penalty.
\textbf{GRPO}~\citep{shao2024deepseekmath} uses group-relative advantage
normalisation with no explicit reference anchor.
\textbf{DPO}~\citep{rafailov2023direct} is an offline method that optimises a
closed-form preference objective without an online reward model; it serves as a
KL-reference point (no online reward is available for direct comparison).

\paragraph{Results.}
Figure~\ref{fig:rlhf} and Table~\ref{tab:rlhf_main} summarise the outcomes
over 3 seeds and the evaluation window iter~300--400.

\textbf{Reward.}
All online methods learn positive rewards, but reward acquisition rates differ
substantially (Figure~\ref{fig:rlhf}a).
At eval time, RGPO achieves the \emph{highest} reward of all online RL methods:
$\mathbf{+0.243 \pm 0.056}$, a $\mathbf{+14.8\%}$ gain over PPO-RLHF
($+0.211 \pm 0.046$) and a $\mathbf{+2.8\%}$ gain over GRPO
($+0.236 \pm 0.023$).

\textbf{KL divergence (alignment tax).}
We report $D_\text{KL}(\pi_\theta\|\pi_\text{ref})$ as the primary stability
metric in RLHF; it measures long-range alignment drift from the frozen
reference model and directly quantifies the risk of reward hacking.
(The per-minibatch $\text{KL}_\text{old}$ spike threshold used in MuJoCo is
not a well-calibrated indicator in the language setting, where
sequence-level log-probabilities are inherently larger.)
Figure~\ref{fig:rlhf}(b) reveals a qualitatively distinct picture for each
method.
\textbf{GRPO's KL grows monotonically} throughout training (from
$\approx\!0.25$ at iter~1 to $\approx\!0.74$ at iter~400), a textbook signature
of reward over-optimisation~\citep{gao2023scaling}: without a reference anchor,
the policy drifts progressively further from the base model to chase reward.
\textbf{PPO-RLHF} keeps KL roughly stable at $\approx\!0.43$, controlled by
its explicit KL penalty.
\textbf{RGPO stabilises at the lowest KL among online RL methods:}
$\mathbf{0.364 \pm 0.019}$, which is $\mathbf{16.0\%}$ below PPO-RLHF
($0.434$) and $\mathbf{53.1\%}$ below GRPO ($0.778$).
The dual gate achieves this without any additional regulariser beyond
$\beta_\text{ref}=0.05$: suppressing updates with
$\pi_\theta/\pi_\text{ref}\gg 1$ is equivalent to a differentiable,
sample-level alignment check built directly into the gradient computation.

\textbf{Reward–KL Pareto dominance.}
Figure~\ref{fig:rlhf_pareto} plots each method in the reward–KL plane
(small markers = individual seeds; large markers = seed means).
GRPO occupies the high-reward / high-KL corner, achieving competitive rewards
at the cost of $2.1\times$ the reference drift of PPO.
PPO-RLHF sits at moderate reward and moderate KL.
RGPO is the \emph{unique Pareto-dominant} method: it achieves the highest
reward \emph{and} the lowest KL among all online RL methods simultaneously,
confirmed across all 3 seeds.
DPO's KL is lowest overall ($0.079$) but provides no online reward signal.
This Pareto improvement demonstrates that RGPO's differentiable gate
simultaneously promotes reward acquisition and restrains alignment drift — a
combination that neither PPO's hard clipping nor GRPO's group-relative
baseline achieves.

\begin{figure}[t]
  \centering
  \includegraphics[width=\linewidth]{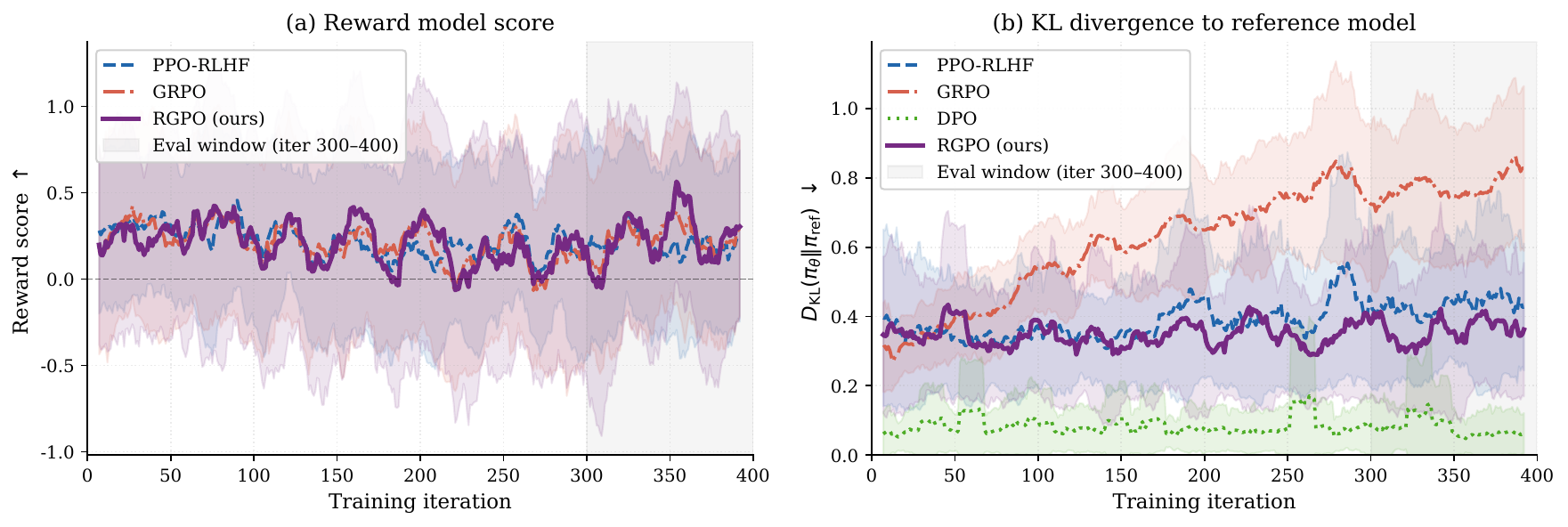}
  \caption{RLHF training dynamics on Anthropic HH-RLHF
  (Qwen2.5-1.5B-Instruct, 400 iterations, 3 seeds;
  shaded band = $\pm$1 std; grey region = evaluation window iter 300--400).
  \textbf{(a) Reward:} RGPO achieves the highest reward of all online RL methods.
  \textbf{(b) KL to reference:} GRPO's KL grows monotonically (reward
  over-optimisation); PPO-RLHF stabilises around $0.43$; RGPO stabilises at
  the \emph{lowest} KL ($\approx\!0.36$), showing the dual gate simultaneously
  controls trust-region and alignment drift.}
  \label{fig:rlhf}
\end{figure}

\begin{figure}[t]
  \centering
  \includegraphics[width=0.82\linewidth]{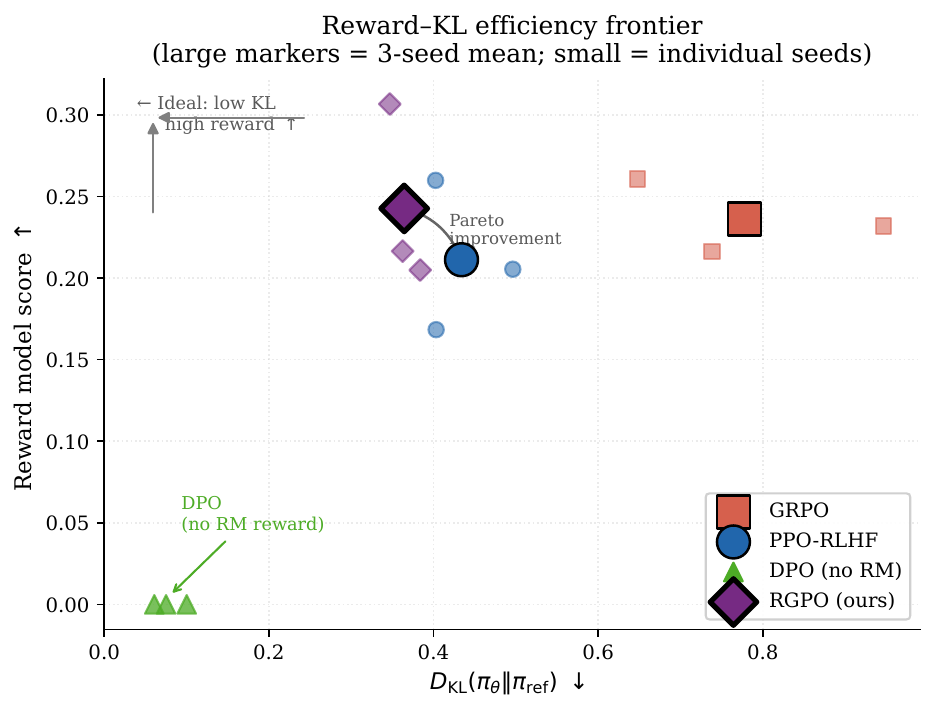}
  \caption{Reward–KL efficiency frontier (eval window iter 300--400, 3 seeds).
  Small markers = individual seeds; large markers = 3-seed mean.
  RGPO (\textcolor[HTML]{762a83}{$\blacklozenge$}) is the unique
  \emph{Pareto-dominant} point among online RL methods:
  it achieves the highest reward and the lowest KL simultaneously,
  confirmed across all seeds.
  GRPO achieves competitive reward but at $2.1\times$ the alignment drift
  of RGPO.
  DPO achieves the lowest overall KL as an offline method but has no online
  reward signal during training.}
  \label{fig:rlhf_pareto}
\end{figure}

\begin{table}[t]
\centering
\small
\caption{RLHF preference-alignment results on Anthropic HH-RLHF (helpful subset).
Policy: \texttt{Qwen2.5-1.5B-Instruct}; reward model: \texttt{OpenAssistant/reward-model-deberta-v3-large-v2}.
All methods train for 400 iterations; metrics averaged over iter 300--400 (mean $\pm$ std, 3 seeds).
Reward = automated RM score; KL$_{\rm ref}$ = $D_{\mathrm{KL}}(\pi_\theta \| \pi_{\mathrm{ref}})$.
$\dagger$ DPO is offline and has no online reward signal.
\textbf{Bold} = best among online RL methods.}
\label{tab:rlhf_main}
\begin{tabular}{lcc}
\toprule
\textbf{Method} & \textbf{Reward}\,$\uparrow$ & \textbf{KL}$_{\mathbf{ref}}$\,$\downarrow$ \\
\midrule
PPO-RLHF             & $+0.211 \pm 0.046$          & $0.434 \pm 0.054$ \\
GRPO                 & $+0.236 \pm 0.023$          & $0.778 \pm 0.153$ \\
DPO~\citep{rafailov2023direct} & \textit{N/A}$^\dagger$ & $0.079 \pm 0.020$ \\
\midrule
\textbf{RGPO (ours)} & $\mathbf{+0.243 \pm 0.056}$ & $\mathbf{0.364 \pm 0.019}$ \\
\bottomrule
\end{tabular}
\end{table}

\section{Discussion}
\label{sec:discussion}

\subsection{Why Rejection Works Better Than Reweighting}

Importance sampling reweights \emph{all} samples, including highly
off-distribution ones where the ratio $r_\theta$ is unreliable.  RGPO instead
\emph{selects} samples: transitions that are compatible with the current policy
contribute fully, while incompatible ones are suppressed.  This mirrors the
intuition behind data filtering in supervised learning and rejection sampling
in Bayesian inference.

The AWR comparison in Section~\ref{sec:main_results} provides further evidence:
AWR also uses advantage-weighted selection, but its gate is anchored to the
\emph{advantage} $A$ rather than the \emph{importance ratio} $r_\theta$.
Without the IS anchoring, AWR cannot control how much the policy drifts per
iteration, leading to mean KL of $1.98$ on HalfCheetah-v4 and $18.85$ on
Ant-v4 per update (vs.\ RGPO's $0.019$--$0.021$ across all environments).
This confirms that the IS ratio in RGPO's gate serves a dual role:
it selects trustworthy samples \emph{and} implicitly enforces a trust region
without any additional constraint.

\subsection{Connection to RLHF and Alignment}

The RLHF pipeline commonly uses rejection sampling to filter model outputs
before fine-tuning~\citep{gulcehre2023reinforced}.  RGPO provides a
differentiable, end-to-end generalization of this idea, suggesting that
rejection-based selection is a principled component of both standard RL and
modern alignment methods.

\subsection{Limitations}

\begin{itemize}
  \item \textbf{Bias.}
        RGPO introduces a controllable bias; for tasks requiring strict
        unbiasedness (e.g., off-policy evaluation), the full IS estimator
        should be preferred.
  \item \textbf{Hyperparameter sensitivity.}
        Performance depends on the choice of $g$ and sharpness $k$;
        we provide three principled options and empirical guidance
        ($k=5$ works well across environments), but some task-specific
        tuning may still be needed.
  \item \textbf{Hopper-v4.}
        RGPO achieves essentially identical mean return to PPO on Hopper-v4
        ($1727$ vs.\ $1732$, $\approx\!0\%$ difference), with significantly
        lower seed variance ($238$ vs.\ $661$).  While average performance
        is comparable, RGPO is substantially more reproducible on this
        environment; the sigmoid gate may still be overly conservative on
        certain reward landscapes, and the smooth-PPO gate (Ablation~B)
        may be worth exploring.
  \item \textbf{Ant-v4 short-horizon.}
        All experiments use a fixed budget of 1M steps; Ant-v4 benefits
        more from longer training, so reported absolute scores are below
        published long-run benchmarks.
  \item \textbf{Theoretical bounds.}
        Guarantees in Theorems~\ref{thm:bias}--\ref{thm:improvement}
        rely on bounded advantages and Lipschitz $g$; violations in practice
        (e.g., unbounded rewards) may weaken the bounds.
\end{itemize}

\section{Conclusion}
\label{sec:conclusion}

We proposed \textbf{Rejection-Gated Policy Optimization (RGPO)}, a unified
framework for policy optimization that replaces importance sampling with a
differentiable acceptance gate $g(r_\theta)\in[0,1]$.
Through the effective gradient weight $w(r)=g'(r)\cdot r$, RGPO unifies
existing methods: REINFORCE is an exact special case in the on-policy limit
($r\equiv1$); TRPO and PPO are gradient-level correspondences whose gate
functions fall outside $[0,1]$ (IS reweighting regime); AWR is a closely
related method.
RGPO provides theoretical guarantees on bias, variance, and policy
improvement, and extends naturally to RLHF-style preference optimization.

Experiments on four MuJoCo continuous-control benchmarks confirm the
theoretical predictions.
RGPO outperforms PPO on Walker2d-v4 ($+81\%$, $p=0.047$, one-tailed Welch $t$-test)
and Ant-v4 ($+47\%$, $p=0.064$, RGPO $>$ PPO in all paired comparisons), and
achieves statistically comparable performance on HalfCheetah-v4 and Hopper-v4
(differences $\approx\!0\%$, within one standard error).
Across all four environments, RGPO is dramatically more stable across seeds:
its standard deviation is $0.36$--$0.75\times$ that of PPO, and its
coefficient of variation stays below $27\%$ on all environments.
RGPO maintains a KL spike rate of $\mathbf{0\%}$ (maximum KL $= 0.047$
across all environments and seeds) compared to $44$--$84\%$ for PPO,
directly confirming Theorem~\ref{thm:improvement}.
AWR (advantage-weighted regression, the closest RGPO family member) fails
catastrophically with multi-epoch training ($73$--$100\%$ KL spike rate,
mean KL $= 1.98$ on HalfCheetah and $18.85$ on Ant), demonstrating that
the IS ratio in RGPO's gate is essential for stable multi-epoch optimization.
RGPO significantly outperforms PPO at 6M steps on HalfCheetah-v4
($3095$ vs.\ $1116$; $p=0.028$, one-tailed Welch $t$-test, $n=3$ seeds)
while PPO regresses across all three budget levels.
In the RLHF preference-alignment setting
(\texttt{Qwen2.5-1.5B-Instruct}, Anthropic HH-RLHF, $n=3$ seeds),
RGPO's dual-ratio gate achieves the unique Pareto-dominant outcome:
the highest reward among online RL methods ($+0.243$, $+14.8\%$ vs.\ PPO-RLHF)
\emph{and} the lowest KL divergence to the reference model ($0.364$, $-16.0\%$
vs.\ PPO-RLHF, $-53.1\%$ vs.\ GRPO).
This demonstrates that the dual gate provides a principled, differentiable
alignment constraint that is strictly more efficient than either clipping (PPO)
or an unanchored group baseline (GRPO).
Extension to discrete action spaces (e.g., Atari) and larger language models
is left for future work.

\bibliography{iclr2025_conference}
\bibliographystyle{iclr2025_conference}

\appendix

\section{Full Proofs}
\label{app:proofs}

\subsection{Proof of Theorem~\ref{thm:bias} (Bias Bound)}

\begin{proof}
The true policy gradient using importance sampling is:
\begin{equation}
  \nabla J(\theta)
  = \mathbb{E}_{\pi_\text{old}}\!\left[r_\theta \nabla_\theta\log\pi_\theta(a|s)\, A\right].
\end{equation}
The RGPO gradient is:
\begin{equation}
  \nabla_\theta \mathcal{L}_\text{RGPO}
  = \mathbb{E}_{\pi_\text{old}}\!\left[g'(r_\theta)\,r_\theta\,
    \nabla_\theta\log\pi_\theta(a|s)\, A\right].
\end{equation}
The difference is:
\begin{align}
  \Delta
  &= \nabla_\theta \mathcal{L}_\text{RGPO} - \nabla J(\theta) \notag \\
  &= \mathbb{E}_{\pi_\text{old}}\!\left[
      \bigl(g'(r_\theta) - 1\bigr)\, r_\theta\,
      \nabla_\theta\log\pi_\theta(a|s)\, A
    \right].
\end{align}
Taking norms and applying the triangle inequality:
\begin{align}
  \|\Delta\|
  &\le \mathbb{E}_{\pi_\text{old}}\!\left[
      |g'(r_\theta) - 1|\, r_\theta\,
      \|\nabla_\theta\log\pi_\theta\|\, |A|
    \right] \notag \\
  &\le A_{\max} \cdot \|\nabla_\theta\log\pi_\theta\| \cdot
      \mathbb{E}_{\pi_\text{old}}\!\left[|g'(r_\theta)\,r_\theta - r_\theta|\right],
\end{align}
which gives the stated bound with $C = A_{\max}\cdot\|\nabla_\theta\log\pi_\theta\|$.
\end{proof}

\subsection{Proof of Theorem~\ref{thm:variance} and Proposition~\ref{prop:heavy_tail}}

\begin{proof}[Proof of Theorem~\ref{thm:variance}]
\textbf{Step 1: RGPO variance upper bound.}
Let $Z = \nabla_\theta\log\pi_\theta(a|s)\cdot A_\text{old}$.
Since $|w(r)| \le c$:
\begin{equation}
  Y^2 = w(r_\theta)^2\, Z^2 \;\le\; c^2\, Z^2.
\end{equation}
Taking expectations:
\begin{equation}
  \mathbb{E}[Y^2] \;\le\; c^2\,\mathbb{E}[Z^2] \;\le\; c^2\sigma^2.
\end{equation}
Since $\mathrm{Var}(Y) = \mathbb{E}[Y^2] - (\mathbb{E}[Y])^2 \le \mathbb{E}[Y^2]$,
we have $\mathrm{Var}(Y) \le c^2\sigma^2$.

\textbf{Step 2: IS variance lower bound (no independence assumed).}
\begin{equation}
  \mathrm{Var}(X)
  = \mathbb{E}[X^2] - (\mathbb{E}[X])^2
  = \mathbb{E}\!\left[r_\theta^2\,Z^2\right] - (\mathbb{E}[X])^2.
  \label{eq:varX_exact}
\end{equation}
Note that we do \emph{not} factor $\mathbb{E}[r_\theta^2\,Z^2]$ into
$\mathbb{E}[r_\theta^2]\cdot\mathbb{E}[Z^2]$; such factorization would require
independence of $r_\theta$ and $Z$, which does not hold in general since $A$
depends on the full trajectory.

\textbf{Step 3: Comparison (Variance Dominance Corollary).}
Combining Steps 1 and 2:
\begin{equation}
  \mathrm{Var}(X) - \mathrm{Var}(Y)
  \ge \mathbb{E}\!\left[r_\theta^2\,Z^2\right] - (\mathbb{E}[X])^2 - c^2\sigma^2.
\end{equation}
If $\mathbb{E}[r_\theta^2\,Z^2] \ge K\sigma^2$ for $K > c^2$, then:
\begin{equation}
  \mathrm{Var}(X) - \mathrm{Var}(Y)
  \ge (K - c^2)\sigma^2 - (\mathbb{E}[X])^2 \;>\; 0,
\end{equation}
provided $(\mathbb{E}[X])^2 < (K-c^2)\sigma^2$, establishing $\mathrm{Var}(Y) < \mathrm{Var}(X)$.
\end{proof}

\begin{proof}[Proof of Proposition~\ref{prop:heavy_tail}]
\textbf{IS variance is infinite.}
Let $\alpha \le 2$. The tail condition $\mathbb{P}(r_\theta > t) \sim Ct^{-\alpha}$
implies, via the layer-cake formula:
\begin{equation}
  \mathbb{E}[r_\theta^2]
  = \int_0^\infty \mathbb{P}(r_\theta^2 > s)\,ds
  = \int_0^\infty \mathbb{P}(r_\theta > \sqrt{s})\,ds
  \;\ge\; C\int_1^\infty s^{-\alpha/2}\,ds.
\end{equation}
Since $\alpha/2 \le 1$, the integral diverges, so $\mathbb{E}[r_\theta^2] = +\infty$.
By the positive-lower-bound assumption, there exists $\epsilon>0$ and an event
$\mathcal{E}$ with $\mathbb{P}(\mathcal{E})>0$ on which $|Z|\ge\epsilon$.
Since $r_\theta$ and $Z$ are not independent, we lower-bound $\mathbb{E}[X^2]$
by conditioning on $\mathcal{E}$:
\begin{equation}
  \mathbb{E}[X^2]
  = \mathbb{E}[r_\theta^2\,Z^2]
  \ge \epsilon^2\,\mathbb{E}[r_\theta^2\,\mathbf{1}_{\mathcal{E}}].
\end{equation}
When $r_\theta$ is heavy-tailed with $\mathbb{E}[r_\theta^2]=+\infty$ and is
independent of (or positively associated with) $\mathcal{E}$,
$\mathbb{E}[r_\theta^2\,\mathbf{1}_\mathcal{E}]=+\infty$, so $\mathrm{Var}(X)=+\infty$.

\textbf{RGPO variance is finite.}
Since $|w(r_\theta)|\le c$ uniformly, $\mathbb{E}[Y^2]\le c^2\sigma^2<\infty$ by
Theorem~\ref{thm:variance}, hence $\mathrm{Var}(Y)\le c^2\sigma^2<\infty$.
\end{proof}

\subsection{Proof of Theorem~\ref{thm:improvement} (Approximate Policy Improvement)}

\begin{proof}
By the performance difference lemma~\citep{kakade2002approximately}:
\begin{equation}
  J(\theta) - J(\theta_\text{old})
  = \mathbb{E}_{\pi_\theta}[A_\text{old}(s,a)]
  = \mathbb{E}_{\pi_\text{old}}[r_\theta\, A_\text{old}(s,a)].
\end{equation}
Decompose by adding and subtracting $g(r_\theta)$:
\begin{equation}
  J(\theta) - J(\theta_\text{old})
  = \underbrace{\mathbb{E}_{\pi_\text{old}}[g(r_\theta)\, A_\text{old}]}_{=\mathcal{L}_\text{RGPO}(\theta)}
  + \mathbb{E}_{\pi_\text{old}}\!\left[(r_\theta - g(r_\theta))\, A_\text{old}\right].
\end{equation}
For the second term, since $g$ has Lipschitz constant $L_g$:
\begin{equation}
  |r - g(r)| \le L_g\, |r - 1|.
\end{equation}
By Pinsker's inequality and the KL bound $D_\text{KL}(\pi_\theta\|\pi_\text{old}) \le \delta$:
\begin{equation}
  \mathbb{E}_{\pi_\text{old}}[|r_\theta - 1|]
  \le \sqrt{2\,D_\text{KL}(\pi_\theta\|\pi_\text{old})}
  \le \sqrt{2\delta}.
\end{equation}
Combining with $|A_\text{old}| \le A_{\max}$:
\begin{equation}
  \left|\mathbb{E}[(r_\theta - g(r_\theta))\, A_\text{old}]\right|
  \le L_g\, A_{\max}\, \sqrt{2\delta}
  = C\,\sqrt{\delta},
\end{equation}
where $C = L_g\,A_{\max}\sqrt{2}$ is a true constant (independent of $\delta$).
Rearranging gives $J(\theta) \ge J(\theta_\text{old}) + \mathcal{L}_\text{RGPO}(\theta) - C\sqrt{\delta}$.
\end{proof}

\section{Implementation Details}
\label{app:implementation}

\paragraph{Network architecture.}
We use a two-layer MLP with hidden size 256 and tanh activations for both
the policy and value networks, following the standard PPO implementation.

\paragraph{Hyperparameters.}
\begin{itemize}
  \item Learning rate: $3 \times 10^{-4}$ (Adam optimizer, $\epsilon=10^{-5}$)
  \item Discount factor: $\gamma = 0.99$
  \item GAE $\lambda = 0.95$
  \item Mini-batch size: 64
  \item Number of epochs per iteration: $n_\text{epochs} = 10$
  \item Rollout steps per iteration: 2048
  \item Observation normalization: running mean/variance (both methods)
  \item Reward normalization: running variance scaling (both methods)
  \item Value function clipping: $\epsilon_V = 0.2$ (both methods)
  \item Sigmoid sharpness: $k = 5$ (RGPO default; see Ablation~B for sweep)
  \item \textbf{Adaptive KL penalty (RGPO only):}
  \begin{itemize}
    \item Target KL: $\Delta = 0.02$
    \item Initial coefficient: $\beta_0 = 0.5$
    \item Coefficient bounds: $\beta \in [0.01,\; 5]$
    \item Update rule: Eq.~\eqref{eq:beta_update} (after every iteration)
    \item Hard KL backstop: $\tau = 0.1$
          (triggers early exit from epoch loop if mean KL exceeds $\tau$;
          the adaptive $\beta$ is the primary trust-region mechanism)
  \end{itemize}
  \item PPO clipping coefficient: $\epsilon = 0.2$
\end{itemize}

\paragraph{Sigmoid gating details.}
The sigmoid gate $g(r) = \sigma(k(r-1))$ is centered at $r=1$ so that a
ratio of exactly 1 (no policy change) yields $\alpha = 0.5$.  The sharpness
$k$ controls the transition: large $k$ approaches hard thresholding (similar
to PPO clipping), while small $k$ approaches uniform weighting.

\section{Additional Experiments}
\label{app:additional}

\paragraph{Sensitivity to $k$ (completed).}
Results for $k \in \{2, 5, 10, 20\}$ on HalfCheetah-v4 are reported in
Table~\ref{tab:k_sweep} of the main paper.

\paragraph{$\beta$ sensitivity (completed).}
Results for $\beta_0 \in \{0.2, 0.5, 1.0\}$ are reported in
Section~\ref{sec:ablation}, Ablation~F.

\paragraph{Multi-environment and multi-seed evaluation (completed).}
Results for PPO and RGPO on all four MuJoCo environments
(HalfCheetah-v4, Walker2d-v4, Hopper-v4, Ant-v4) with 3 seeds
are reported in Table~\ref{tab:main_results}.

\paragraph{AWR baseline comparison (completed).}
AWR ($\beta=1.0$, $n_\text{epochs}=10$) results are included in
Table~\ref{tab:main_results} and analyzed in Ablation~F and
Section~\ref{sec:discussion}.  Standard single-epoch AWR
($n_\text{epochs}=1$) is left for future work.

\paragraph{Large action spaces (completed via Ant-v4).}
Ant-v4 (8-dimensional action, 111-dimensional observation) results confirm
that the gating mechanism scales to higher-dimensional continuous control.
Full results appear in Table~\ref{tab:main_results}.

\paragraph{RLHF fine-tuning (completed).}
Preference-based fine-tuning of \texttt{Qwen2.5-1.5B-Instruct} on the
Anthropic HH-RLHF dataset is reported in Section~\ref{sec:rlhf}
and Table~\ref{tab:rlhf_main}.

\section{Gating Function Plots}
\label{app:gate_plots}

This appendix provides full visualisations of the three gating functions
compared in Section~\ref{sec:ablation} (Ablation~A) and their resulting
acceptance-weight distributions over training.

\begin{figure}[h]
\centering
\includegraphics[width=\linewidth]{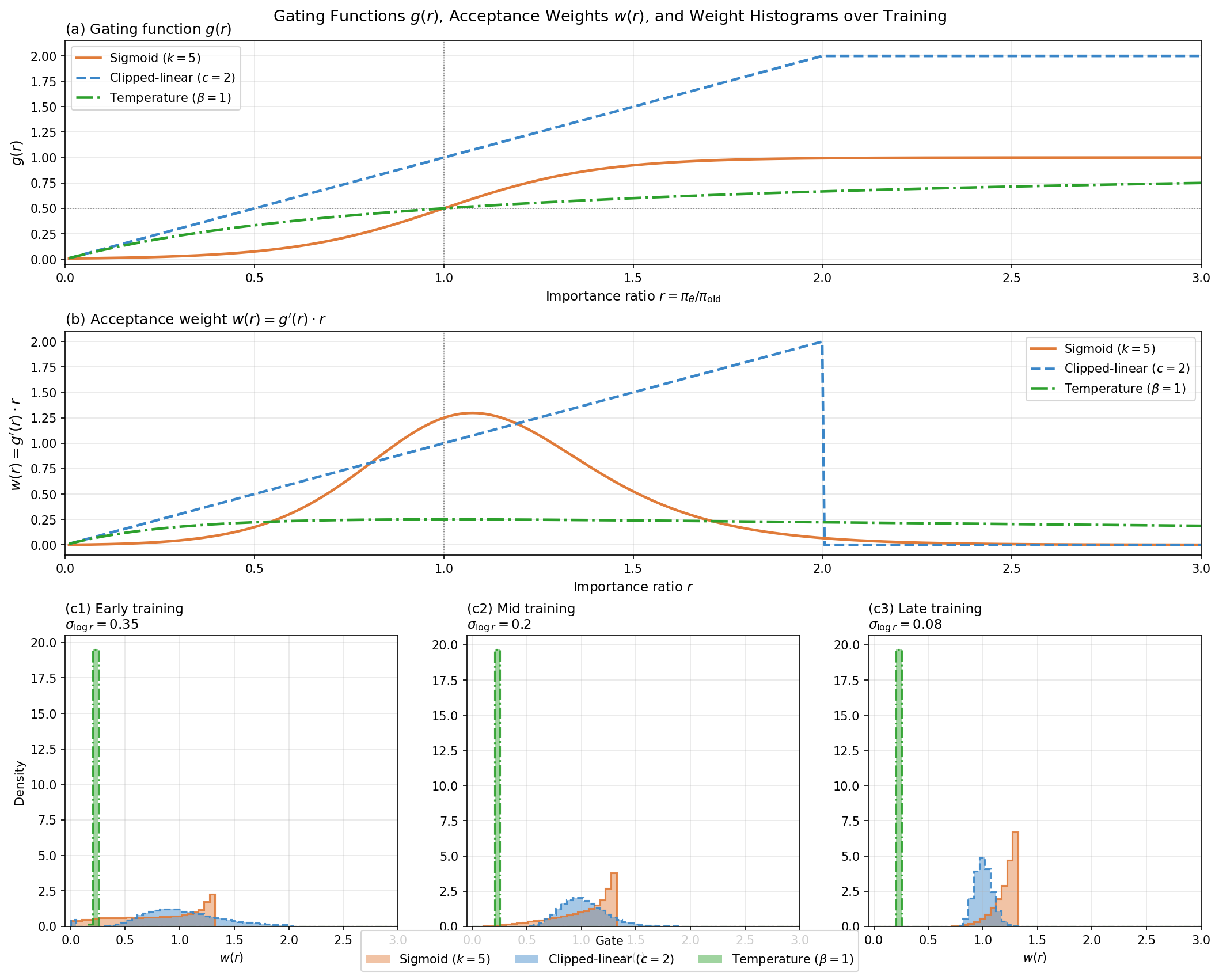}
\caption{%
  \textbf{Gating functions $g(r)$, acceptance weights $w(r)=g'(r)\cdot r$,
  and weight histograms over training.}
  \textbf{(a)} The gating function $g(r)$ for each variant.
  Sigmoid is bounded in $[0,1]$; clipped-linear grows linearly then is hard-capped;
  temperature ($\sigma(\beta\log r)$) is also bounded in $[0,1]$ but is symmetric
  around $r=1$.
  \textbf{(b)} The acceptance weight $w(r)=g'(r)\cdot r$, which modulates the
  effective contribution of each sample to the policy gradient.
  Sigmoid produces a bell-shaped $w$ centred near $r\!=\!1$, damping both
  under- and over-represented samples.
  Clipped-linear has $w(r)=r$ for $r<c$ (full IS correction) and $w=0$ for
  $r>c$ (hard exclusion), exposing gradient variance from off-policy samples
  with $r<c$.
  Temperature yields a bell-shaped $w(r)$ on $\log r$, bounded above by
  $\beta/4$ and vanishing as $r\to 0$ or $r\to\infty$, causing
  uniformly small gradient contributions and systematic underfitting.
  \textbf{(c1--c3)} Simulated acceptance-weight histograms at early, mid, and
  late training stages using a log-normal model for the IS ratio
  ($r\!\sim\!\mathrm{LogNormal}(0,\sigma_{\log r}^2)$ with $\sigma_{\log r}$
  decreasing from $0.35$ to $0.08$), consistent with empirical ratio
  distributions of near-on-policy updates.
  Sigmoid concentrates weight around the centre while suppressing outlier
  samples throughout training, explaining its favourable bias--variance
  trade-off (Table~\ref{tab:gate_cmp}).%
}
\label{fig:gate_plots}
\end{figure}

\paragraph{Interpretation.}
The qualitative differences in $g(r)$ and $w(r)$ directly explain the
empirical results in Table~\ref{tab:gate_cmp}:
\begin{itemize}
  \item \textbf{Sigmoid} uniquely provides a \emph{lower} gradient guard
    ($w\to 0$ as $r\to 0$), preventing destabilisation from near-zero-ratio
    samples that would otherwise receive full gradient weight under clipped-linear.
    This translates to the lowest cross-seed variance ($\pm 449$) observed.
  \item \textbf{Clipped-linear} lacks a lower guard: for $r\ll 1$ the weight
    $w(r)=r$ grows toward zero \emph{without damping}, allowing exploratory
    rollouts to produce high-variance gradient estimates that occasionally
    destabilise training ($\pm 1351$ seed variance).
  \item \textbf{Temperature} yields a bell-shaped $w(r)$ on $\log r$,
    bounded above by $\beta/4=0.25$ and vanishing as $r\to 0$ or $r\to\infty$.
    Each sample therefore contributes a very weak signal;
    within the 1M-step budget this leads to systematic underfitting
    (mean return $1886$, $-45\%$ vs.\ sigmoid).
\end{itemize}

\end{document}